\documentclass[11pt]{article}

\usepackage{fancybox}

\usepackage{color}
\usepackage{url}
\usepackage[margin=1in]{geometry}

\usepackage{comment}
\usepackage{amsmath,amssymb,amsthm,bm,mathtools}
\usepackage{dsfont,multirow,hyperref,setspace,enumerate}

\usepackage[authoryear]{natbib}
\hypersetup{colorlinks,linkcolor={blue},citecolor={blue},urlcolor={blue}} 
\usepackage{algpseudocode,algorithm}
\algnewcommand\algorithmicinput{\textbf{Input:}}
\algnewcommand\algorithmicoutput{\textbf{Output:}}
\algnewcommand\INPUT{\item[\algorithmicinput]}
\algnewcommand\OUTPUT{\item[\algorithmicoutput]}

\mathtoolsset{showonlyrefs=true}

\theoremstyle{plain}
\newtheorem{thm}{Theorem}
\newtheorem{lem}{Lemma}
\newtheorem{prop}{Proposition}

\newtheorem{cor}{Corollary}

\theoremstyle{definition}

\newtheorem{exmp}{Example}
\newtheorem{rmk}{Remark}

\usepackage{stmaryrd}
\def\entry#1{\llbracket #1 \rrbracket}
\def\Vec{\operatorname{vec}}

\def\ma{\bm{a}}
\def\mb{\bm{b}}

\def\me{\bm{e}}

\def\mx{\bm{x}}

\def\mA{\bm{A}}
\def\mB{\bm{B}}

\def\mD{\bm{D}}

\def\mH{\bm{H}}
\def\mI{\bm{I}}

\def\mL{\bm{L}}
\def\mM{\bm{M}}
\def\mN{\bm{N}}
\def\mO{\bm{O}}

\def\mS{\bm{S}}

\def\tA{\mathcal{A}}
\def\tB{\mathcal{B}}
\def\tC{\mathcal{C}}
\def\tD{\mathcal{D}}
\def\tE{\mathcal{E}}

\def\tH{\mathcal{H}}

\def\tL{\mathcal{L}}
\def\tM{\mathcal{M}}

\def\tO{\mathcal{O}}

\def\tS{\mathcal{S}}

\def\tX{\mathcal{X}}
\def\tY{\mathcal{Y}}

\def\trueT{\Theta_{\text{true}}}

\newcommand{\vectornorm}[1]{\left\lVert#1\right\rVert_2}

\newcommand{\zeronormSize}[2]{#1\lVert#2#1\rVert_0}
\newcommand{\nnormSize}[2]{#1\lVert#2#1\rVert_*}

\newcommand{\normSize}[2]{#1\lVert#2#1\rVert_\sigma}
\newcommand{\FnormSize}[2]{#1\lVert#2#1\rVert_F} 
\newcommand{\mnormSize}[2]{#1\lVert#2#1\rVert_\text{max}} 

\DeclareMathOperator*{\argmax}{arg\,max}

\def\MLET{\hat \Theta_{\text{MLE}}}
\newcommand*{\KeepStyleUnderBrace}[1]{%f
  \mathop{%
    \mathchoice
    {\underbrace{\displaystyle#1}}%
    {\underbrace{\textstyle#1}}%
    {\underbrace{\scriptstyle#1}}%
    {\underbrace{\scriptscriptstyle#1}}%
  }\limits
}

\begin{document}

\begin{center}
\begin{spacing}{1.5}
\textbf{\Large Learning from Binary Multiway Data: Probabilistic Tensor Decomposition and its Statistical Optimality }
\end{spacing}

\vspace{5mm}
Miaoyan Wang$^{1*}$\let\thefootnote\relax\footnote{$^*$To whom correspondence should be addressed: miaoyan.wang@wisc.edu} and Lexin Li$^{2}$

\vspace{5mm}
$^1$Department of Statistics, University of Wisconsin-Madison\\
$^2$Department of Biostatistics and Epidemiology, University of California, Berkeley 
\end{center}

\vspace{.5cm}

\begin{abstract}
We consider the problem of decomposing a higher-order tensor with binary entries. Such data problems arise frequently in applications such as neuroimaging, recommendation system, topic modeling, and sensor network localization. We propose a multilinear Bernoulli model, develop a rank-constrained likelihood-based estimation method, and obtain the theoretical accuracy guarantees. In contrast to continuous-valued problems, the binary tensor problem exhibits an interesting phase transition phenomenon according to the signal-to-noise ratio. The error bound for the parameter tensor estimation is established, and we show that the obtained rate is minimax optimal under the considered model. Furthermore, we develop an alternating optimization algorithm with convergence guarantees. The efficacy of our approach is demonstrated through both simulations and analyses of multiple data sets on the tasks of tensor completion and clustering. 
\end{abstract}


\section{Introduction}

\subsection{Motivation}
Multiway arrays have gained increasing attention in numerous fields, such as genomics~\citep{hore2016tensor}, neuroscience~\citep{zhou2013tensor}, recommender systems~\citep{bi2018recom}, social networks~\citep{nickel2011three}, and computer vision~\citep{tang2013tensor}. An important reason of the wide applicability is the effective representation of data using tensor structure. One example is recommender system~\citep{bi2018recom}, the data of which can be naturally described as a three-way tensor of user $\times$ item $\times$ context and each entry indicates the user-item interaction under a particular context. Another example is the DBLP database~\citep{zhe2016distributed}, which is organized into a three-way tensor of author $\times$ word $\times$ venue and each entry indicates the co-occurrence of the triplets. 

Despite the popularity of continuous-valued tensors, recent decades have witnessed many instances of binary tensors, in which all tensor entries are binary indicators encoded as 0/1. Examples include click/no-click action in recommender systems, presence/absence of edges in multi-relational social networks~\citep{nickel2011three}, and connection/disconnection in brain structural connectivity networks~\citep{wang2019common}. These binary tensors are often noisy and high-dimensional. It is crucial to develop effective tools that reduce the dimensionality, take into account the tensor formation, and learn the underlying structures of these massive discrete observations. A number of successful tensor decomposition methods have been proposed~\citep{kolda2009tensor, anandkumar2014tensor, wang2017tensor}, revitalizing the classical methods such as CANDECOMP/PARAFAC (CP) decomposition~\citep{hitchcock1927expression} and Tucker decomposition~\citep{tucker1966some}. These methods treat tensor entries as continuous-valued, and therefore they are not suitable to analyze binary tensors. 
%There is a relative paucity of decomposition methods for binary tensors compared to continuous-valued tensors.

In this article, we develop a general method and the associated theory for binary tensor decomposition. Let $\tY = \entry{y_{i_1,\ldots,i_K}}\in\{0,1\}^{d_1\times \cdots \times d_K}$ be an order-$K$ $(d_1,\ldots,d_K)$-dimensional binary data tensor, where the entries $y_{i_1,\ldots,i_K}$ are either 1 or 0 that encodes the presence or absence of the event indexed by the $K$-tuplet $(i_1,\ldots,i_K)$. We consider the following low-rank Bernoulli model,
\begin{equation}\label{eq:modelintro}
\tY|\Theta\sim \text{Bernoulli}\{ f(\Theta) \},\quad  \text{where} \quad \text{rank}(\Theta)=R,
\end{equation}
where, for ease of notation, we have allowed the operators ($\sim$, $f$, etc) to be applied to tensors in an element-wise manner. That is, the entries of $\tY$ are realizations of independent Bernoulli random variables with success probability $f(\theta_{i_1,\ldots,i_K})$, where $f$ is a suitable function that maps $\mathbb{R}$ to $[0,1]$. The parameter tensor, $\Theta = \entry{\theta_{i_1,\ldots,i_K}}$ is of the same dimension as $\tY$ but its entries are continuous-valued, and we assume $\Theta$ admits a low-rank CP structure. Our goal is to estimate $\Theta$ from one instance of the binary tensor $\tY$. In particular, we are interested in the high dimensional setting where $d_{\min}=\min_{k\in[K]}d _k$ grows. Our primary focus is to understand (i) the statistical estimation error of binary tensor decomposition; (ii) the statistical hardness, in terms of minimax rate and signal-to-noise ratio, of the binary problem compared to its continuous-valued counterpart; and (iii) the computational properties of associated estimation algorithms.

\subsection{Related Work}
\label{sec:relatedwork}

Our work is closely related to but also clearly distinctive from several lines of existing research. We survey the main related approaches for comparison. 

\emph{Continuous-valued tensor decomposition.} In principle, one can apply the existing decomposition methods designed for continuous-valued tensor~\citep{kolda2009tensor, wang2017tensor} to binary tensor, by pretending the 0/1 entries were continuous. However, such an approach will yield an inferior performance: flipping the entry coding $0\leftrightarrow 1$ would totally change the decomposition result, and the predicted values for the unobserved entries could fall outside the valid range $[0,1]$. Our method, in contrast, is invariant to flipping, because reversing the entry coding of $\tY$ changes only the sign but not the decomposition result of the parameter $\Theta$. Moreover, as we show in Section~\ref{sec:phase}, binary tensor decomposition exhibits a ``dithering'' effect~\citep{davenport2014} that necessitates the presence of stochastic noise in order to estimate $\Theta$. This is clearly contrary to the behavior of continuous-valued tensor decomposition.  

\emph{Binary matrix decomposition.} When the order $K=2$, the problem reduces to binary or logit principal component analysis (PCA), and a similar model as \eqref{eq:modelintro} has been proposed~\citep{collins2002generalization,de2006principal,lee2010sparse}. While tensors are conceptual generalization of matrices, matrix decomposition and tensor decomposition are fundamentally different~\citep{kolda2009tensor}. Under the matrix case, the rank $R$ is required to be no greater than $\min(d_1,d_2)$, and the factor matrices are constrained to be orthogonal for the identification purpose. Both constraints are unnecessary for tensors, since the uniqueness of tensor CP decomposition holds under much milder conditions~\citep{bhaskara2014uniqueness}. In fact, factors involved in tensors may be nonorthogonal, and the tensor rank $R$ may exceed the dimension. These differences make the earlier algorithms built upon matrix decomposition unsuitable to tensors. Moreover, as we show in Section~\ref{sec:upperbound}, if we were to apply the matrix version of binary decomposition to a tensor by unfolding the tensor into a matrix, the result is suboptimal with a slower convergence rate.

\emph{Binary tensor decomposition.} More recently, \cite{mavzgut2014dimensionality,rai2015scalable,hong2020generalized} studied higher-order binary tensor decomposition, and we target the same problem. However, our study differs in terms of the scope of the results. In general, there are two types of properties that an estimator possesses. The first type is the algorithm-dependent property that quantifies the impact of a specific algorithm, such as the choice of loss function, initialization, and iterations, on the final estimator. The second type is the statistical property that characterizes the population behavior and is independent of any specific algorithm. Earlier solutions of \cite{mavzgut2014dimensionality,rai2015scalable,hong2020generalized} focused on the algorithm effectiveness, but did not address the population optimality. By contrast, we study both types of properties in Sections \ref{sec:statistic-property} and \ref{sec:algorithm-property}. This allows us to better understand the gap between a specific algorithm and the population optimality, which may in turn offer a useful guide to the algorithm design.

\emph{1-bit completion.} Our work is also connected to 1-bit matrix completion~\citep{cai2013max, davenport2014} and its recent extension to 1-bit tensor completion~\citep{ghadermarzy2018learning}. The completion problem aims to recover a matrix or tensor from incomplete observations of its entries. The observed entries are highly quantized, sometimes even to a single bit. We first show in Section~\ref{sec:latent} that our Bernoulli tensor model has an equivalent interpretation as the threshold model commonly used in 1-bit quantization. Then, the two methods are compared in Section~\ref{sec:upperbound}. We achieve a faster convergence rate than that in 1-bit tensor completion~\citep{ghadermarzy2018learning}, assuming the signal rank is of constant order. The optimality of our estimator is safeguarded by a matching minimax lower bound.

\emph{Boolean tensor decomposition.} Boolean tensor decomposition~\citep{miettinen2011boolean, erdos2013discovering, rukat2018probabilistic} is a data-driven algorithm that decomposes a binary tensor into binary factors. The idea is to use logical operations to replace arithmetic operations such as addition and multiplication in the factorization. These methods also study binary tensors, same as we do, but they took an empirical approach to approximate a particular data instance. One important difference is that we focus on parameter estimation in a population model. The population interpretation offers useful insight on the effectiveness of dimension reduction. Having a population model allows us to tease apart the algorithmic error versus the statistical error. We numerically compare the two approaches in Section~\ref{sec:empirical}.

\emph{Bayesian binary tensor decomposition.} There have been a number of Bayesian binary tensor decomposition algorithms~\citep{nickel2011three, rai2014scalable, rai2015scalable}. Most of these algorithms focus on the specific context of multi-relational learning. Although we take multi-relational learning as one of our applications, we address a general binary tensor decomposition problem, and we study the statistical properties of the problem, such as the SNR phase diagram and minimax rate. Besides, we provide a frequentist-type solution which is computationally more tractable than a Bayesian one.

\subsection{Our Contributions}
\label{sec:contributions}

The primary goal of this paper is to study both the statistical and computational properties of binary tensor problem. Our contributions are summarized below.

First, we quantify the differences and connections between binary tensor problem and continuous-valued tensor problem. We show that the Bernoulli tensor model \eqref{eq:modelintro} is equivalent to entrywise quantization of a latent noisy, continuous-valued tensor. The impact of latent signal-to-noise ratio (SNR) on the tensor recovery accuracy is characterized, and we identify three different phases for tensor recovery according to SNR; see Table~\ref{tab:comparison} in Section \ref{sec:phase}. When SNR is bounded by a constant, the loss in binary tensor decomposition is comparable to the case of continuous-valued tensor, suggesting very little information has been lost by quantization. On the other hand, when SNR is sufficiently large, stochastic noise turns out to be helpful, and is in fact essential, for estimating the signal tensor. The later effect is related to ``dithering''~\citep{davenport2014} and ``perfect separation''~\citep{albert1984existence} phenomenon, and this is clearly contrary to the behavior of continuous-valued tensor decomposition. 

Second, we propose a method for binary tensor decomposition and establish its statistical properties, including the upper bound and the minimax lower bound on the tensor recovery accuracy. These properties characterize the population optimality of the estimator. Note that, in our problem, the tensor dimensions $(d_1, \ldots, d_K)$ diverge, and so does the number of unknown parameters. As such, the classical maximum likelihood estimation (MLE) theory does not directly apply. We leverage the recent development in random tensor theory and high-dimensional statistics to establish the error bounds of the tensor estimation. The matching information-theoretical lower bounds are correspondingly provided. To our knowledge, these statistical guarantees are among the first for binary tensor decomposition. 

Lastly, we %supplement the above general statistical properties by 
propose an alternating optimization algorithm for binary tensor decomposition and establish the %corresponding 
algorithmic convergence. Our algorithm-dependent error bound reveals an interesting interplay between statistical and computational efficiency. We illustrate the efficacy of our algorithm through both simulations and data applications.

\subsection{Notation and Organization}
\label{sec:organization}

We adopt the following notation throughout the article. We use $\tY=\entry{y_{i_1,\ldots,i_K}}\in\mathbb{F}^{d_1\times \cdots \times d_K}$ to denote an order-$K$ $(d_1,\ldots,d_K)$-dimensional tensor over a filed $\mathbb{F}$. We focus on real or binary tensors, i.e., $\mathbb{F}=\mathbb{R}$ or $\mathbb{F}=\{0,1\}$. The Frobenius norm of $\tY$ is defined as $\FnormSize{}{\tY}=(\sum_{i_1,\ldots,i_K}y_{i_1,\ldots,i_K}^2)^{1/2}$, and the maximum norm of $\tY$ is defined as $\mnormSize{}{\tY}=\max_{i_1,\ldots,i_K}|y_{i_1,\ldots,i_K}|$. We use uppercase letters (e.g., $\Theta$, $\tY$, $\mA$) to denote tensors and matrices, and use lowercase letters (e.g.,\ $\theta$, $\ma$) to denote scales and vectors. The vectorization of tensor $\tY$, denoted $\text{vec}(\tY)$, is defined as the operation rearranging all elements of $\tY$ into a column vector. We use $\ma\otimes \mb$ to denote the kronecker product of vectors $\ma$ and $\mb$, and $\mA\odot\mB$ for the Khatri-Rao product of matrices $\mA$ and $\mB$. We use $\mS^{d-1} = \{\mx \in \mathbb{R}^d\colon \vectornorm{\mx} = 1\}$ to denote the $(d-1)$-dimensional unit sphere, and the shorthand $[n]:=\{1,...,n\}$ to denote the $n$-set for $n \in \mathbb{N}_{+}$. 

The rest of the article is organized as follows. Section~\ref{sec:model} presents the low-rank Bernoulli tensor model, its connection with 1-bit observation model, and the rank-constrained MLE framework. In Section~\ref{sec:statistic-property}, we establish the statistical estimation error bounds and the phase transition phenomenon. We next develop an alternating optimization algorithm and establish its convergence guarantees in Section~\ref{sec:algorithm-property}. We present the simulations in Section~\ref{sec:simulations} and data analyses in Section~\ref{sec:realdata}. All technical proofs are deferred to Section~\ref{sec:proofs} and Appendix~\ref{sec:lemma}. We conclude the paper with a discussion in Section~\ref{sec:conclusion}.

\section{Model}
\label{sec:model}

\subsection{Low-rank Bernoulli Model}
\label{sec:structure}

Let $\tY=\entry{y_{i_1,\ldots,i_K}}\in\{0,1\}^{d_1\times \cdots \times d_K}$ be a binary data tensor. We assume the tensor entries are realizations of independent Bernoulli random variables, such that, for all $(i_1,\ldots,i_K)\in[d_1]\times \cdots \times [d_K]$, 
\begin{equation}\label{eq:model}
\mathbb{P}(y_{i_1,\ldots,i_K}=1)=f(\theta_{i_1,\ldots,i_K}). 
\end{equation}
In this model, $f\colon \mathbb{R}\to[0,1]$ is a strictly increasing function. We further assume that $f(\theta)$ is twice-differentiable in $\theta\in\mathbb{R}/\{0\}$; $f(\theta)$ is strictly increasing and strictly log-concave; and $f'(\theta)$ is unimodal and symmetric with respect to $\theta=0$. All these assumptions are fairly mild. In the context of generalized linear models (GLMs), $f$ is often referred to as the ``inverse link function.'' When no confusion arises, we also call $f$ the ``link function.'' The parameter tensor $\Theta=\entry{\theta_{i_1,\ldots,i_K}}\in\mathbb{R}^{d_1\times \cdots \times d_K}$ is continuous-valued and unknown; it is the main object of interest in our tensor estimation inquiry. The entries of $\tY$ are assumed to be mutually independent conditional on $\Theta$, which is commonly adopted in the literature~\citep{collins2002generalization, de2006principal, lee2010sparse}. Note that this assumption does not rule out the marginal correlations among the entries of $\tY$. 

Furthermore, we assume the parameter tensor $\Theta$ admits a rank-$R$ CP decomposition,
\begin{equation}\label{eq:cp}
\Theta=\sum_{r=1}^R\lambda_r\ma^{(1)}_r\otimes \cdots \otimes \ma^{(K)}_r,
\end{equation}
where $\lambda_1 \geq \ldots \geq \lambda_R>0$ and $\ma^{(k)}_r\in\mathbf{S}^{d_k-1}$, for all $r\in[R]$, $k\in[K]$. Without loss of generality, we assume that $\Theta$ cannot be written as a sum of fewer than $R$ outer products. The CP structure in \eqref{eq:cp} is frequently used in tensor data analysis, and the rank $R$ determines the tradeoff between model complexity and model flexibility. 
For the theory, we assume the true rank $R$ is known; the adaptation to unknown $R$ is addressed in Section~\ref{sec:miss-rank}. The low-rank structure dramatically reduces the number of  parameters in $\Theta$, from the order of $\prod_{k}d_k$ to the order of $\sum_{k}d_k$. More precisely, the effective number of parameters in \eqref{eq:cp} is $p_e=R\left(d_1+d_2\right)-R^2$ for matrices ($K=2$) after adjusting for the nonsingular transformation indeterminacy, and $p_e=R\left(\sum_k d_k-K+1\right)$ for higher-order tensors ($K\geq 3$) after adjusting for the scaling indeterminacy.  

Combining \eqref{eq:model} and \eqref{eq:cp} leads to our low-rank Bernoulli model. We seek to estimate the rank-$R$ tensor $\Theta$ given the observed binary tensor $\tY$. The model can be viewed as a generalization of the classical CP decomposition for continuous-valued tensors to binary tensors, in a way that is analogous to the generalization from a linear model to a GLM. When imposing low-rank structure to a continuous-valued tensor $\tY$ directly, the problem amounts to seeking the best rank-$R$ approximation to $\tY$, in the least-squares sense. The  least-squares criterion is equivalent to the MLE for the low-rank tensor $\Theta$ based on a noisy observation $\tY = \Theta + \tE$, where $\tE \in \mathbb{R}^{d_1\times \cdots \times d_k}$ collects independent and identically distributed (i.i.d.) Gaussian noises. In the next section, we present a close connection between a continuous-valued tensor problem and a binary tensor problem.

\subsection{Latent Variable Model Interpretation} 
\label{sec:latent}

We show that our binary tensor model~\eqref{eq:model} has an equivalent interpretation as the threshold model commonly used in 1-bit quantization~\citep{davenport2014,bhaskar20151,cai2013max,ghadermarzy2018learning}. The later viewpoint sheds light on the nature of the binary (1-bit) measurements from the information perspective.  

Consider an order-$K$ tensor $\Theta=\entry{\theta_{i_1,\ldots,i_K}}\in\mathbb{R}^{d_1\times \cdots \times d_K}$ with a rank-$R$ CP structure. Suppose that we do not directly observe $\Theta$. Instead, we observe the quantized version $\tY=\entry{y_{i_1,\ldots,i_K}}\in\{0,1\}^{d_1\times \cdots \times d_K}$ following the scheme 
\begin{equation}\label{latent}
y_{i_1,\ldots,i_K}=\begin{cases}
1 & \text{if } \theta_{i_1,\ldots,i_K}+\varepsilon_{i_1,\ldots,i_K}\geq 0,\\
0 & \text{if } \theta_{i_1,\ldots,i_K}+\varepsilon_{i_1,\ldots,i_K}<0,
\end{cases}
\end{equation}
where $\tE=\entry{\varepsilon_{i_1,\ldots,i_K}}$ is a noise tensor to be specified later. Equivalently, the observed binary tensor is $\tY=\text{sign}(\Theta+\tE)$, and the associated latent tensor is $\Theta+\tE$. Here the sign function $\text{sign}(x)\stackrel{\text{def}}{=}\mathds{1}_{\{x\geq 0\}}$ is applied to tensors in an element-wise manner. In light of this interpretation, the tensor $\Theta$ serves as an underlying, continuous-valued quantity whose noisy discretization gives $\tY$.

The latent model~\eqref{latent} in fact is equivalent to our Bernoulli tensor model \eqref{eq:model}, if the link $f$ behaves like a cumulative distribution function. Specifically, for any choice of $f$ in \eqref{eq:model}, if we define $\tE$ as having i.i.d.\ entries drawn from a distribution whose cumulative distribution function is $\mathbb{P}(\varepsilon< \theta)=1-f(-\theta)$, then \eqref{eq:model} reduces to \eqref{latent}. Conversely, if we set the link function $f(\theta)=\mathbb{P}(\varepsilon \geq -\theta)$, then model \eqref{latent} reduces to \eqref{eq:model}. Such relationship gives a one-to-one correspondence between the error distribution in the latent model and the link function in the Bernoulli model. We describe three common choices of $f$, or equivalently, the distribution of $\tE$. 

\begin{exmp} (Logistic link/Logistic noise). The logistic model is represented by~\eqref{eq:model} with $f(\theta)=\left(1+e^{-\theta/\sigma}\right)^{-1}$ and the scale parameter $\sigma>0$. Equivalently, the noise $\varepsilon_{i_1,\ldots,i_K}$ in~\eqref{latent} follows i.i.d.\ logistic distribution with the scale parameter $\sigma$. 
\end{exmp}

\begin{exmp} (Probit link/Gaussian noise). The probit model is represented by~\eqref{eq:model} with $f(\theta) = \Phi(\theta/\sigma)$, where $\Phi$ is the cumulative distribution function of a standard Gaussian. Equivalently, the noise $\varepsilon_{i_1,\ldots,i_K}$ in \eqref{latent} follows i.i.d.\ $N(0,\sigma^2)$.
\end{exmp}

\begin{exmp} (Laplacian link/Laplacian noise). The Laplacian model is represented by~\eqref{eq:model} with
\begin{equation}
f(\theta)=\begin{cases}
{1\over 2} \exp \left({\theta \over \sigma}\right), & \text{if } \theta < 0,\\
1-{1\over 2}\exp(-{\theta \over \sigma}),& \text{if } \theta \geq 0,
\end{cases}
\end{equation}
and the scale parameter $\sigma>0$. Equivalently, the noise $\varepsilon_{i_1,\ldots,i_K}$ in \eqref{latent} follows i.i.d.\ Laplace distribution with the scale parameter $\sigma$. 
\end{exmp}

\noindent
The above link functions are common for the Bernoulli model, and the choice is informed by several considerations~\citep{mccullagh1980regression}. The probit is the canonical link based on the Bernoulli likelihood, and it has a direct connection with the log-odds of success. The probit is connected to threshold latent Gaussian tensors. The Laplace has a heavier tail than the normal distribution, and it is more suitable for modeling long-tail data.

\subsection{Rank-constrained Likelihood-based Estimation}

We propose to estimate the unknown parameter tensor $\Theta$ in model~\eqref{eq:model} using a constrained likelihood approach. The log-likelihood function for \eqref{eq:model} is
\begin{align}
\tL_{\tY}(\Theta)  &=  \sum_{i_1,\dots,i_K}\left[ \mathds{1}_{\{ y_{i_1,\ldots,i_K}=1\}}\log f(\theta_{i_1,\ldots,i_K}) 
+ \mathds{1}_{\{ y_{i_1,\ldots,i_K}=0\}}\log \left\{1-f(\theta_{i_1,\ldots,i_K})\right\} \right] \\
&=  \sum_{i_1,\ldots,i_K}\log f\left[(2y_{i_1,\ldots,i_K}-1)\theta_{i_1,\ldots,i_K}\right], 
\end{align}
where the second equality is due to the symmetry of the link function $f$. To incorporate the CP structure \eqref{eq:cp}, we propose a constrained optimization,
\begin{align}\label{eq:MLE}
\MLET=\argmax_{\Theta\in\tD}\tL_{\tY}(\Theta), \quad \text{where} \;\;
\tD \subset \tS=\left\{ \Theta\colon \text{rank} (\Theta)=R, \textrm{ and } \mnormSize{}{\Theta} \leq \alpha  \right\},
\end{align}
for a given rank $R\in\mathbb{N}_{+}$ and a bound $\alpha\in\mathbb{R}_{+}$. Here the search space $\tD$ is assumed to be a compact set containing the true parameter $\Theta_{\text{true}}$. The candidate tensor of our interest satisfies two constraints. The first is that $\Theta$ admits the CP structure \eqref{eq:cp} with rank $R$. As discussed in Section~\ref{sec:structure}, the low-rank structure \eqref{eq:cp} is an effective dimension reduction tool in tensor data analysis. The second constraint is that all the entries of $\Theta$ are bounded in absolute value by a constant $\alpha\in\mathbb{R}_{+}$. We refer to $\alpha$ as the ``signal'' bound of $\Theta$.  This maximum-norm condition is a technical assumption to aid the recovery of $\Theta$ in the noiseless case. Similar techniques have been employed for the matrix case~\citep{davenport2014,bhaskar20151,cai2013max}.

In the next section, we first investigate the statistical error bounds for the global optimizer $\MLET$. These bounds characterize the population behavior of the global estimator and weave three quantities: tensor dimension, rank, and signal-to-noise ratio. We then compare these properties to the information-theoretical bound and reveal a phase-transition phenomenon. In Section~\ref{sec:algorithm-property}, we develop a specific algorithm for the optimization problem in \eqref{eq:MLE}, and we derive the convergence properties of the empirical estimator resulting from this algorithm.

\section{Statistical Properties}
\label{sec:statistic-property}

\subsection{Performance Upper Bound}
\label{sec:upperbound}

We define two quantities $L_\alpha$ and $\gamma_\alpha$ to control the ``steepness'' and ``convexity'' of the link function $f$. Let 
\begin{equation}
L_\alpha=\sup_{|\theta|\leq \alpha}\left\{{ \dot{f}(\theta) \over f(\theta) \left(1-f(\theta) \right) }\right\}, \quad\text{and}\quad
\gamma_\alpha = \inf_{|\theta|\leq \alpha}\left\{ {\dot{f}^2 (\theta)\over f^2(\theta)} -{\ddot{f}(\theta)\over f(\theta)} \right\},
\end{equation}
where $\dot{f}(\theta)=df(\theta)/d\theta$, and $\alpha$ is the bound on the entrywise magnitude of $\Theta$. When $\alpha$ is a fixed constant and $f$ is a fixed function, all these quantities are bounded by some fixed constants independent of the tensor dimension. In particular, for the logistic, probit and Laplacian models, we have 
\begin{align}
\textrm{Logistic model: } \quad & L_\alpha={ 1\over \sigma}, \quad \gamma_\alpha={ e^{\alpha/\sigma} \over (1+e^{\alpha/\sigma})^2 \sigma^2}, \\
\textrm{Probit model: } \quad & L_\alpha\leq {2\over \sigma}\left({\alpha\over \sigma}+1\right), \quad \gamma_\alpha\geq {1\over \sqrt{2\pi}\sigma^2} \left({\alpha\over \sigma}+{1\over 6}\right)e^{-x^2/\sigma^2},\\
\textrm{Laplacian model: }\quad & L_\alpha \leq {2\over \sigma},\quad \gamma_\alpha\geq {e^{-\alpha/\sigma}\over 2\sigma^2}.
\end{align}

We assess the estimation accuracy using the deviation in Frobenius norm. For the true coefficient tensor $\trueT\in\mathbb{R}^{d_1\times\cdots\times d_K}$ and its estimator $\hat\Theta$, define 
\begin{equation}
\text{Loss}(\hat \Theta, \trueT) = {1\over \sqrt{\prod_{k}d_k}}\FnormSize{}{\hat \Theta-\trueT}.
\end{equation}

The next theorem establishes the upper bound for $\MLET$ under model \eqref{eq:model}.

\begin{thm}[Statistical convergence] \label{thm:rate}
Suppose $\tY\in\{0,1\}^{d_1\times\dots\times d_K}$ is an order-$K$ binary tensor following model~\eqref{eq:model} with the link function $f$ and the true coefficient tensor $\trueT\in\tD$. Let $\MLET$ be the constrained MLE in~\eqref{eq:MLE}. Then, there exists an absolute constant $C_1>0$, and a constant $C_2>0$ that depends only on $K$, such that, with probability at least $1-\exp\left( -C_1 \log K \sum_k d_k \right)$,
\begin{equation}\label{eq:bound}
\text{Loss}(\MLET, \trueT) \leq \min\left( 2\alpha,\ {C_2 L_\alpha\over  \gamma_\alpha} \sqrt{ R^{K-1} \sum_kd_k\over  \prod_k d_k} \right). 
\end{equation}
\end{thm}

\noindent
Note that $f$ is strictly log-concave if and only if $\ddot{f}(\theta)f(\theta)< \dot{f}(\theta)^2$~\citep{boyd2004convex}. Henceforth, $\gamma_\alpha>0$ and $L_\alpha>0$, which ensures the validity of the bound in \eqref{eq:bound}. 

In fact, the proof of Theorem~\ref{thm:rate} (see Section~\ref{sec:proofs}) shows that the statistically optimal rate holds, not only for the MLE $\hat \Theta_{\text{MLE}}$, but also for any estimators $\hat \Theta$ in the level set $\left\{ \hat\Theta \in \tD\colon \tL_\tY(\hat \Theta)\geq \tL_{\tY}(\trueT) \right\}$. 

To compare our upper bound to existing results in literature, we consider a special setting where the dimensions are the same in all modes; i.e., $d_1=\cdots=d_K=d$. In such a case, our bound \eqref{eq:bound} reduces to
\begin{equation}\label{eq:ours}
\text{Loss}(\MLET, \trueT) \leq \tO\left( 1 \over d^{(K-1)/2}\right), \text{ as }d\to \infty,
\end{equation}
for a fixed rank $R$ and a fixed signal bound $\alpha$. The MLE thus achieves consistency with polynomial convergence rate. Our bound has a faster convergence rate than that in 1-bit tensor recovery~\citep{ghadermarzy2018learning},  
\begin{equation}
\text{Loss}(\hat \Theta, \trueT) \leq \tO\left(1 \over d^{(K-1)/4}\right), \text{ as }d\to \infty.
\end{equation}
The rate improvement comes from the fact that we impose an exact low-rank structure on $\Theta$, whereas \cite{ghadermarzy2018learning} employed the max norm as a surrogate rank measure. 

Our bound also generalizes the previous results on low-rank binary matrix completion. The convergence rate for rank-constrained matrix completion is $\tO(1/\sqrt{d})$~\citep{bhaskar20151}, which fits into our special case when $K=2$. Intuitively, in the tensor data analysis problem, we can view each tensor entry as a data point, and sample size is the total number of entries. A higher tensor order has a larger number of data points and thus exhibits a faster convergence rate as $d\to \infty$.  

We compare the results~\eqref{eq:ours} to the scenario if we apply the matrix version of binary decomposition to a tensor by unfolding the tensor into a matrix. The ``best'' matricization solution that unfolds a tensor into a near-square matrix~\citep{mu2014square} gives a convergence rate $\tO(d^{-{\lfloor{K/2\rfloor} \over 2}})$, with $\lfloor{K/2\rfloor}$ being the integer part of $K/2$. The gap between the rates highlights the importance of decomposition that specifically takes advantage of the multi-mode structure in tensors.

As an immediate corollary of Theorem~\ref{thm:rate}, we obtain the explicit form of the upper bound \eqref{eq:bound} when the link $f$ is a logistic, probit, or Laplacian function. 

\begin{cor}
Assume the same setup as in Theorem~\ref{thm:rate}. There exists an absolute constant $C'>0$ such that with probability at least $1-\exp\left( -C' \log K \sum_k d_k \right)$,
\begin{equation} \label{eq:probit}
\text{Loss}(\MLET, \trueT) \leq  \min\left\{2\alpha,\ C(\sigma, \alpha) \sqrt{ R^{K-1} \sum_kd_k\over  \prod_k d_k}\right\},
\end{equation}
where $C(\alpha,\sigma)$ is a scaler factor,
\[
C(\alpha, \sigma)=\begin{cases}
C_1 \sigma \left( 2+ e^{\alpha\over \sigma}+ e^{-{\alpha\over \sigma}}\right)& \text{for the logistic link},\\
C_2 \sigma\left(\displaystyle {\alpha+\sigma\over 6\alpha+\sigma}\right)e^{\alpha^2 \over \sigma^2}& \text{for the probit link},\\
C_3 \alpha e^{ \alpha \over \sigma} & \text{for the Laplacian link},
\end{cases}
\]
and $C_1, C_2, C_3>0$ are constants that depend only on $K$.
\end{cor}
\noindent
The dependency of the above error bounds on the signal bound $\alpha$ and the noise level $\sigma$ will be discussed in Section~\ref{sec:phase}.

\subsection{Information-theoretical Lower Bound}
\label{sec:information}

We next establish two lower bounds. The first lower bound is for all statistical estimators $\hat\Theta$, including but not limited to the estimator $\MLET$ in~\eqref{eq:MLE}, under the binary tensor model \eqref{eq:model}. The result is based on the information theory and is thus algorithm-independent. We show that this lower bound nearly matches the upper bound on the estimation accuracy of $\MLET$, thereby implying the rate optimality of $\MLET$. 

With a little abuse of notation, we use $\tD(R,\alpha)$ to denote the set of tensors with the  rank bounded by $R$ and the maximum norm bounded by $\alpha$. The next theorem establishes this first lower bound for all estimators $\hat\Theta$ in $\tD(R,\alpha)$ under the model~\eqref{eq:model}.

\begin{thm}[Minimax lower bound for binary tensors]\label{thm:minimax}
Suppose $\tY \in \{0,1\}^{d_1\times\dots\times d_K}$ is an order-$K$ binary tensor generated from the model $\tY=\text{sign}(\trueT+\tE)$, where $\trueT\in\tD(R,\alpha)$ is the true parameter tensor and $\tE$ is a noise tensor of i.i.d.\ Gaussian entries. Suppose that $R\leq \min_k d_k$ and the dimension $\max_k d_k \geq 8$. Let $\text{inf}_{\hat \Theta}$ denote the infimum over all estimators $\hat \Theta \in \tD(R, \alpha)$ based on the binary tensor observation $\tY$. Then, there exist absolute constants $\beta_0\in(0,1)$ and $c_0>0$, such that
\begin{equation}\label{eq:lower}
\inf_{\hat \Theta }\sup_{\trueT\in\tD(R,\alpha)} \mathbb{P}\left\{ \text{Loss}(\hat \Theta, \trueT)  \geq c_0 \min\left( \alpha, \ \sigma{\sqrt{Rd_{\max} \over \prod_k d_k}} \right) \right\} \geq \beta_0.
\end{equation}
\end{thm}

\noindent 
Here we only present the result for the probit model, while similar results can be obtained for the logistic and Laplacian models. In this theorem, we assume that $R \leq \min_kd_k$. This condition is automatically satisfied in the matrix case, since the rank of a matrix is always bounded by its row and column dimension. For the tensor case, this assertion may not always hold. However, in the most applications, the tensor rank is arguably smaller than its dimension. We view this as a mild condition. Note that the earlier Theorem~\ref{thm:rate} places no constraint on the rank $R$. In Section~\ref{sec:empirical}, we will assess the empirical performance when the rank exceeds dimension.   

We next compare the lower bound~\eqref{eq:lower} to the upper bound~\eqref{eq:probit}, as the tensor dimension $d_k \to \infty$ while the signal bound $\alpha$ and the noise level $\sigma$ are fixed. Since $d_{\max}\leq \sum_k d_k \leq Kd_{\max}$, both the bounds are of the form $C \sqrt{d_{\max}} \left( \prod_{k}d_k\right)^{-1/2}$, where $C$ is a factor that does not depend on the tensor dimension. Henceforth, our estimator $\MLET$ is rate-optimal. 

The second lower bound is for all estimators $\tilde\Theta$ based on the ``unquantized" observation $(\Theta+\tE)$, which enables the evaluation of information loss due to binary quantization $\tY=\text{sign}(\Theta+\tE)$. Recall that Section~\ref{sec:latent} introduces a latent variable view of binary tensor model as an entrywise quantization of a noisy continuous-valued tensor. We seek an estimator $\tilde\Theta$ by ``denoising'' the continuous-valued observation $(\Theta+\tE)$. The lower bound is obtained via an information-theoretical argument and is again applicable to all estimators $\tilde\Theta \in \tD(R,\alpha)$. 

\begin{thm}[Minimax lower bound for continuous-valued tensors] \label{thm:ratereal}
Suppose $\tilde \tY \in \mathbb{R}^{d_1\times\dots\times d_K}$ is an order-$K$ continuous-valued tensor generated from the model $\tilde\tY=\trueT+\tE$, where $\trueT\in\tD(R,\alpha)$ is the true parameter tensor and $\tE$ is a noise tensor of i.i.d.\ Gaussian entries. Suppose that $R\leq \min_kd_k$ and $\max_k d_k \geq 8$. Let $\text{inf}_{\hat \Theta}$ denote the infimum over all estimators $\tilde \Theta \in \tD(R, \alpha)$ based on the continuous-valued tensor observation $\tilde \tY$. Then, there exist absolute constants $\beta_0\in(0,1)$ and $c_0>0$ such that
\begin{equation}\label{eq:lowerreal}
\inf_{\tilde \Theta }\sup_{\trueT\in\tD(R,\alpha)}  \mathbb{P}\left\{ \text{Loss}(\tilde \Theta, \trueT)\geq c_0 \min\left( \alpha, \sigma \sqrt{Rd_{\max}\over \prod_k d_k} \right) \right\} \geq \beta_0.
\end{equation}
\end{thm}

\noindent
This lower bound~\eqref{eq:lowerreal} quantifies the statistical hardness of the tensor estimation problem. In the next section, we compare the information loss of tensor estimation, based on the data with quantization, $\text{sign}(\Theta+\tE)$, vs.\ the data without quantization, $(\Theta+\tE)$.

\subsection{Phase Diagram}
\label{sec:phase}

The error bounds we have established depend on the signal bound $\alpha$ and the noise level $\sigma$. In this section, we define three regimes based on the signal-to-noise ratio (SNR) $=\mnormSize{}{\Theta} / \sigma$, in which the tensor estimation exhibits different behaviors. Table~\ref{tab:comparison} and Figure~\ref{fig:snr} summarize the error bounds of the three phrases under the case when $d_1=\cdots=d_K=d$. Our discussion focuses on the probit model, but similar patterns also hold for the logistic and Laplacian models.

\begin{table}[b!]
\centering
\resizebox{\columnwidth}{!}{%
\begin{tabular}{l|ccc} \hline
Tensor type & SNR  $\gg \tO(1) $  & $\tO(1) \gtrsim \text{SNR} \gg \tO(d^{-(K-1)/2}) $& $\tO(d^{-(K-1)/2}) \gtrsim \text{SNR}$\\
\hline
Binary &$\sigma e^{\alpha^2/\sigma ^2} d^{-(K-1)/2}$&$\sigma d^{-(K-1)/2}$&$\alpha$\\
\hline
Continuous & $\sigma d^{-(K-1)/2}$& $\sigma d^{-(K-1)/2}$&$\alpha$\\
\hline
\end{tabular}
}
\caption{Error rate for low-rank tensor estimation. For ease of presentation, we omit the constants that depend on the order $K$ or rank $R$.}
\label{tab:comparison}
\end{table}

\begin{figure}[t]
\centering
\includegraphics[width=13.5cm]{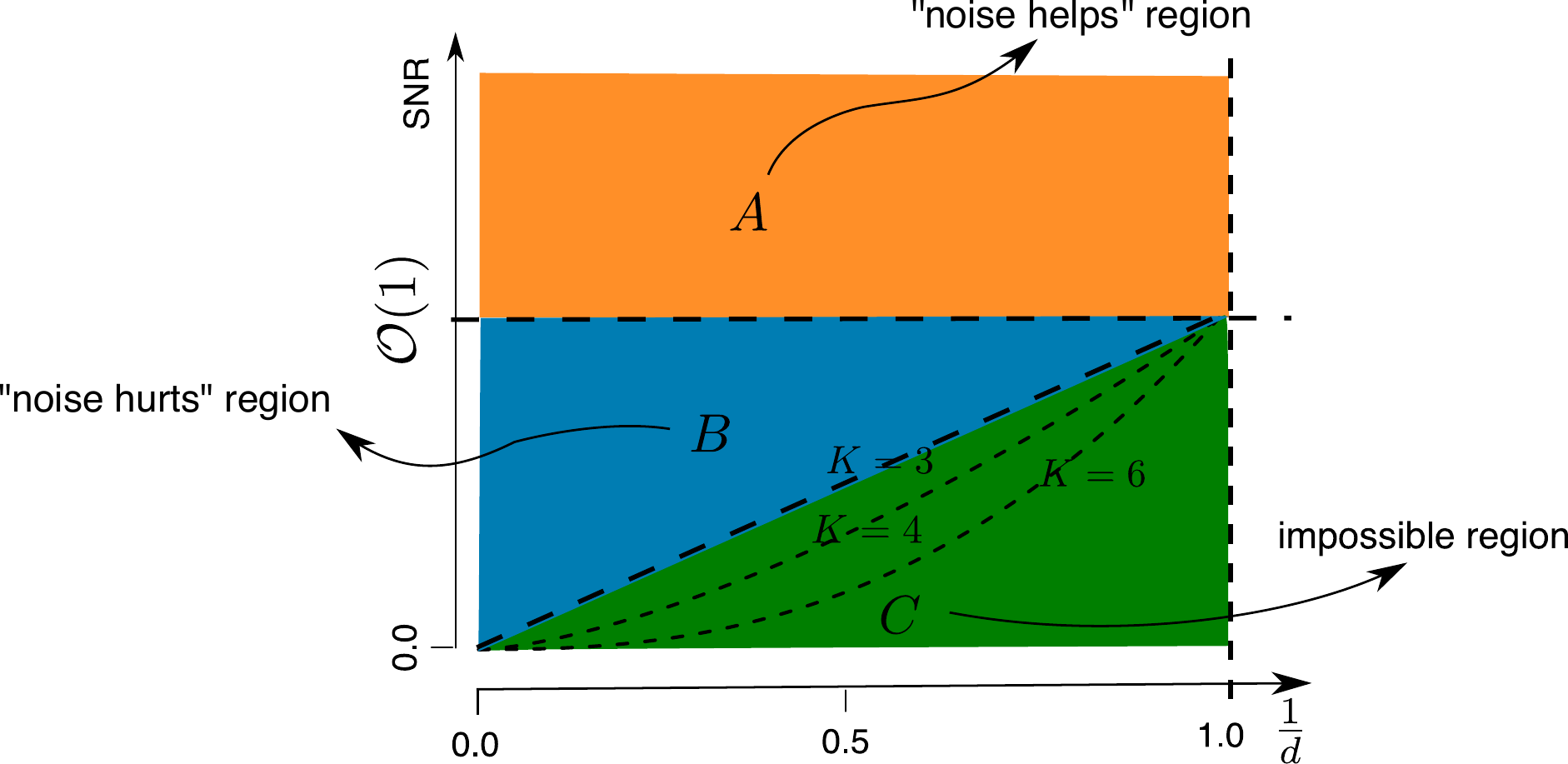}
\caption{Phase diagram according to the SNR. (A) ``Noise helps'' region: the estimation error decreases with the noise . (B) ``Noise hurts'' region: the error increases with the noise. (C) Impossible region: a consistent estimator of $\Theta$ is impossible. The dashed line between regions (B) and (C) depicts the boundary $d^{-(K-1)/2}$ as $K$ varies. Note that the origin in the $x$-axis corresponds to the high-dimensional region, $d^{-(K-1)/2}\to 0$, which is of our main interest.}\label{fig:snr} 
\end{figure}

The first phase is when the noise is weak, in that $\sigma \ll \alpha$ equivalently $\text{SNR} \gg \tO(1)$. In this regime, the error bound in~\eqref{eq:probit} scales as $\sigma \exp({\alpha^2 / \sigma^2})$, suggesting that increasing the noise level would lead to an improved tensor estimation accuracy. This ``noise helps'' region may seem surprising; however it is not an artifact of our proof. It turns out this phenomena is intrinsic to 1-bit quantization, and we confirm this behavior in simulations in Section~\ref{sec:empirical}. As the noise level $\sigma$ goes to zero, the problem essentially reverts to the noiseless case where an accurate estimation of $\Theta$ becomes impossible. To see this, we consider a simple example with a rank-1 signal tensor in the latent model \eqref{latent} in the absence of noise. Two different coefficient tensors, $\Theta_1=\ma_1\otimes \ma_2\otimes \ma_3$ and $\Theta_2=\text{sign}(\ma_1)\otimes \text{sign}(\ma_2)\otimes \text{sign}(\ma_3)$, would lead to the same observation $\tY$, and thus recovery of $\Theta$ from $\tY$ becomes hopeless. Interestingly, adding a stochastic noise $\tE$ to the signal tensor prior to 1-bit quantization completely changes the nature of the problem, and an efficient estimator can be obtained through the likelihood approach. In the 1-bit matrix/tensor completion literature, this phenomenon is referred to as ``dithering'' effect of random noise~\citep{davenport2014}. 

The second phase is when the noise is comparable to the signal, in that $\tO(1) \gtrsim \text{SNR} \gg \tO(d^{-(K-1)/2})$. In this regime, the error bound in~\eqref{eq:probit} scales linearly with $\sigma$. We find that the lower bound~\eqref{eq:lowerreal} from the unquantized tensor matches with the upper bound~\eqref{eq:probit} from a quantized one. This suggests that 1-bit quantization induces very little loss of information towards the estimation of $\Theta$. In other words, $\MLET$, which is based on the quantized observation, can achieve the similar degree of accuracy as if the completely unquantized measurements were observed.

The third phase is when the noise completely dominates the signal, in that $\text{SNR}\lesssim \tO(d^{-(K-1)/2})$. A consistent estimation of $\Theta$ becomes impossible. In this regime, a trivial zero estimator achieves the minimax rate.

\section{Algorithm and Convergence Properties}
\label{sec:algorithm-property}

\subsection{Alternating Optimization Algorithm}
In this section, we introduce an algorithm to solve \eqref{eq:MLE} and study the algorithmic convergence. For notational convenience, we drop the subscript $\tY$ in $\tL_{\tY}(\Theta)$ and simply write $\tL(\Theta)$. The optimization~\eqref{eq:MLE} is a non-convex problem in $\Theta$ due to the non-convexity in the feasible set $\tD$. We use the CP representation of $\Theta$ in~\eqref{eq:cp} and turn the optimization into a block-wise convex problem. Algorithm~\ref{alg:binary} summarizes the full optimization procedure, and we discuss the individual steps in the next paragraph.

\begin{algorithm}[t]
\caption{Binary tensor decomposition}\label{alg:binary}
\begin{algorithmic}[1]
\INPUT Binary tensor $\tY\in\{0,1\}^{d_1\times \cdots \times d_K}$, link function $f$, rank $R$, and entrywise bound $\alpha$.
\OUTPUT Rank-$R$ coefficient tensor $\Theta$, along with the factor matrices $\mA=(\mA_1, \ldots, \mA_K)$. 
  \State Initialize random matrices $\mA^{(0)}=\left\{ \mA^{(0)}_1,\ldots,\mA^{(0)}_K \right\}$ and iteration index $t=0$.
  \While {the relative increase in objective function $\tL(\mA)$ is less than the tolerance} 
    \State Update iteration index $t \leftarrow t + 1$.
    \For {$k$ = 1 to $K$}
    \State Obtain $\mA^{(t+1)}_k$ by solving $d_k$ separate GLMs with link function $f$.
    \EndFor
  \State Line search to obtain $\gamma^*$.
  \State Update $\mA^{(t+1)}_k \leftarrow \gamma^*\mA_k^{(t)}+(1-\gamma^*)\mA_k^{(t+1)}$, for all $k\in[K]$.
  \State Normalize the columns of $\mA^{(t+1)}_k$ to be of unit-norm for all $k\leq K-1$, and absorb the scales into the columns of $\mA^{(t+1)}_K$.
  \EndWhile
\end{algorithmic}
\end{algorithm}

Specifically, write the mode-$k$ factor matrices from \eqref{eq:cp} as 
\begin{equation}\label{eq:convention}
\mA_k=\left[\ma_1^{(k)},\ldots,\ma_R^{(k)}\right] \in \mathbb{R}^{d_k\times R}, \ \text{for }k \in [K-1], \ \textrm{ and } \, \mA_{K}=\left[\lambda_1\ma_1^{(K)},\ldots,\lambda_R\ma_R^{(K)} \right]\in \mathbb{R}^{d_K\times R}, 
\end{equation}
where, without loss of generality, we choose to collect $\lambda_k$'s into the last factor matrix. Let $\mA=(\mA_1,\ldots,\mA_K)$ denote the collection of all block variables satisfying the above convention. Then the optimization problem \eqref{eq:MLE} is equivalent to
\begin{equation}\label{eq:newloss}
\max_{\mA} \tL\{ \Theta(\mA) \}, \;\; \textrm{ subject to } \, \Theta(\mA)\in\tD.
\end{equation}
Although the objective function in \eqref{eq:newloss} is in general not concave in the $K$ factor matrices jointly, the problem is concave in each factor matrix individually with all other factor matrices fixed. This feature enables a block relaxation type minimization, where we alternatively update one factor matrix at a time while keeping the others fixed. In each iteration, the update of each factor matrix involves solving a number of separate GLMs. To see this, let $\mA_k^{(t)}$ denote the $k$th factor matrix at the $t$th iteration, and 
\[
\mA^{(t)}_{-k} = \mA^{(t+1)}_1\odot\cdots\odot \mA^{(t+1)}_{k-1}\odot\mA^{(t)}_{k+1}\odot\cdots\odot\mA^{(t)}_{K}, \quad k=1, \ldots, K.
\]  
Let $\tY(:, j(k), :)$ denote the subtensor of $\tY$ at the $j$th position of the $k$th mode. Then the update $\mA_k^{(t+1)}$ can be obtained row-by-row by solving $d_k$ separate GLMs, where each GLM takes $\text{vec}\{ \tY(:, j(k), :) \} \in \mathbb{R}^{(\prod_{i\neq k}d_i)\times 1}$ as the ``response", $\mA^{(t)}_{-k}\in\mathbb{R}^{(\prod_{i\neq k}d_i)\times R}$ as the ``predictors", and the $j$th row of $\mA_k$ as the ``regression coefficient", for all $j \in[d_k], k \in [K]$. In each GLM, the effective number of predictors is $R$, and the effective sample size is $\prod_{i\neq k}d_i$. These separable, low-dimensional GLMs allow us to leverage the fast GLM solvers as well as parallel processing to speed up the computation. After each iteration, we post-process the factor matrices $\mA^{(t+1)}_k$ by performing a line search,
\begin{equation} 
\gamma^* = \argmax_{\gamma\in[0,1]} \tL_\tY \left\{ \gamma \mA_k^{(t)} + (1-\gamma)\mA_k^{(t+1)} \right\}, \; \textrm{ subject to } \; \mnormSize{}{\Theta}\leq \alpha.
\end{equation}
We then update $\mA^{(t+1)}_k = \gamma^* \mA^{(t)}_k+(1-\gamma^*)\mA^{(t+1)}_k$ and normalize the columns of $\mA^{(t+1)}_k$.

In practice, we run the algorithm from multiple initializations to locate a final estimate with the highest objective value.

\subsection{Algorithmic Properties}\label{sec:algo-convergence}

We study the convergence of Algorithm~\ref{alg:binary}. The convergence of the objective function $\tL$ is guaranteed whenever the $\tL$ is bounded from above, due to the monotonic nature of $\tL$ over iterations. We next study the convergence of the iterates $\mA^{(t)}$ and $\Theta^{(t)}=\Theta\{ \mA^{(t)} \}$. To simplify the analysis, we assume the optimization path is in the interior of the search domain $\{\Theta\colon \mnormSize{}{\Theta}\leq \alpha\}$. We drop the dependence of $\alpha$ for technical convenience, but all the results should be interpreted with this assumption imposed. In practice, $\alpha$ can be adjusted via probing the MLE frontier~\citep{sur2019modern}. One may start with a reasonably large $\alpha$ and check whether MLE is in the interior of the search domain. If perfect separation occurs, one may want to reduce $\alpha$ to a smaller value in order to control the estimation error. We refer to \cite{sur2019modern} for more discussions on adjusting $\alpha$ via probing the MLE frontier.

We need the following assumptions for algorithmic convergence.
\begin{enumerate}[({A}1)]
\item (Regularity condition) The log-likelihood $\tL(\mA)$ is continuous and the set $\{\mA\colon \tL(\mA)\geq \tL(\mA^{(0)})\}$ is compact. \label{ass:1}

\item (Strictly local maximum condition) Each block update in Algorithm~\ref{alg:binary} is well-defined; i.e.,\ the GLM solution exists and is unique, and the corresponding sub-block in the Hession matrix is non-singular at the solution. \label{ass:2}

\item (Local uniqueness condition) The set of stationary points of $\tL(\mA)$ are isolated module scaling. \label{ass:3}

\item (Local Lipschitz condition) Let $\mA^*$ be a local maximizer of $\tL$. The rank-$R$ CP representation $\Theta=\Theta(\mA)$ is locally Lipschitz at $\mA^*$; namely, there exist two constants $c_1,c_2>0$ such that
\[
c_1\FnormSize{}{\mA'-\mA''}\leq \FnormSize{}{\Theta(\mA')-\Theta(\mA'')}\leq c_2\FnormSize{}{\mA'-\mA''},
\]
for $\mA',\mA''$ sufficiently close to $\mA^*$. Here $\mA', \mA''$ represent the block variables subject to convention~\eqref{eq:convention}. 

\end{enumerate}
These conditions are mild and often imposed in the literature. Specifically, Assumption (A1) ensures the upper boundedness of log-likelihood and the existence of global optimum. Therefore, the stopping rule of Algorithm~\ref{alg:binary} is well defined. Assumption (A2) asserts the negative-definiteness of the Hessian in the block coordinate $\mA_k$. Note that the full Hession needs not to be negative-definite in all variables simultaneously. We consider this requirement as a reasonable assumption, as similar conditions have been imposed in various non-convex problems~\citep{uschmajew2012local,zhou2013tensor}.  Assumptions (A2)--(A4) guarantee the local uniqueness of the CP representation $\Theta=\Theta(\mA)$. The conditions exclude the case of rank-degeneracy; e.g.,\ the case when the tensor $\Theta$ can be written in fewer than $R$ factors, or when the columns of $\mA^{(t)}_{-k}$ are linearly dependent in the GLM update.

We comment that the local uniqueness condition is fairly mild for tensors of order three or higher. This property reflects the fundamental difference between tensor and matrix decomposition, in
that the same property often fails for the matrix case. Consider an example of a 2-by-2 matrix. Suppose that the local maximizer is $\Theta^*=\Theta^*(\me_1,\me_2)=\me_1^{\otimes 2}+\me_2^{\otimes 2}$, where $\me_1, \me_2$ are canonical vectors in $\mathbb{R}^2$. The variable $\mA^*=(\me_1,\me_2)$ is a non-attracting point for the matrix problem. Indeed, one can construct a point $\mA^{(0)}=(\ma_1,\ma_2)$, with $\ma_1=(\sin \theta, \cos\theta)'$, and $\ma_2=(\cos \theta, -\sin \theta)'$. The point $\mA^{(0)}$ can be made arbitrarily close to $\mA^{*}$ by tuning $\theta$, but the algorithm iterates initialized from $\mA^{(0)}$ would never converge to $\mA^*$. In contract, a 2-by-2-by-2 tensor problem with the maximizer $\tilde\Theta^*=\tilde\Theta^*(\me_1,\me_2)=\me_1^{\otimes 3}+\me_2^{\otimes 3}$ possesses locally unique decomposition. For more discussion on decomposition uniqueness and its implication in the optimization, we refer to~\cite{kruskal1977three,uschmajew2012local,zhou2013tensor}.

\begin{prop}[Algorithmic convergence] \label{prop:alg}
Suppose Assumptions (A1)-(A3) hold.
\begin{enumerate}[(i)]
\item (Global convergence) Every sequence $\mA^{(t)} = \left\{ \mA^{(t)}_1,\ldots,\mA^{(t)}_K \right\}$ generated by Algorithm~\ref{alg:binary} converges to a stationary point of $\tL(\mA)$.\label{eq:global}
\item (Locally linear convergence) Let $\mA^*$ be a local maximizer of $\tL$. There exists an $\varepsilon$-neighborhood of $\mA^*$, such that, for any staring point $\mA^{(0)}$ in this neighborhood, the iterates $\mA^{(t)}$ of Algorithm~\ref{alg:binary} linearly converge to $\mA^*$,
\begin{equation}
\FnormSize{}{\mA^{(t)}-\mA^*}\leq \rho^t\FnormSize{}{\mA^{(0)}-\mA^*},
\end{equation}
where $\rho\in(0,1)$ is a contraction parameter. Furthermore, if Assumption (A4) holds at $\mA^*$, then there exists a constant $C>0$ such that
\begin{equation}
\FnormSize{}{\Theta(\mA^{(t)})-\Theta(\mA^*)}\leq C\rho^t\FnormSize{}{\Theta(\mA^{(0)})-\Theta(\mA^*)}.
\end{equation}
\label{eq:local}
\vspace{-1cm}
\end{enumerate}
\end{prop}

\noindent
Proposition~\ref{prop:alg}\eqref{eq:local} shows that every local maximizer of $\tL$ is an attractor of Algorithm~\ref{alg:binary}. This property ensures an exponential decay of the estimation error near a local maximum. Combining Proposition~\ref{prop:alg} and Theorem~\ref{thm:rate}, we have the following theorem. 

\begin{thm}[Empirical performance]\label{thm:empirical} 
Let $\tY\in\{0,1\}^{d_1\times \cdots \times d_K}$ be a binary data tensor under the Bernoulli tensor model~\eqref{eq:model} with parameter $\trueT=\Theta(\mA_{\text{true}})$. Let $\mA^{(t)}$ denote a sequence of estimators generated from Algorithm~\ref{alg:binary}, with the limiting point $\mA^*$. Suppose $\mA^*$ is a local maximizer satisfying that $\tL\left( \Theta(\mA^*) \right) \geq \tL(\trueT)$. Furthermore, Assumptions (A1)-(A4) hold. Then, with probability at least $1-\text{exp}(-C'\log K\sum_k d_k)$, there exists an iteration number $T_0\geq0$, such that, 
\begin{equation}\label{eq:empirical}
\text{Loss}\left( \Theta( \mA^{(t)} ),\trueT \right) \leq  \KeepStyleUnderBrace{C_1\rho^{t-T_0}\text{Loss}(\Theta(\mA^{(T_0)}), \trueT)}_{\text{algorithmic error}}+\KeepStyleUnderBrace{{C_2 L_{\alpha}\over \gamma_\alpha}\sqrt{R^{K-1}\sum_k d_k \over \prod_k d_k} }_{\text{statistical error}},
\end{equation}
for all $t\geq T_0$, where $\rho\in (0,1)$ is a contraction parameter, and $C_1,C_2>0$ are two constants.
\end{thm}

\noindent
Theorem~\ref{thm:empirical} provides the estimation error of the empirical estimator from our Algorithm~\ref{alg:binary} at each iteration. The bound~\eqref{eq:empirical} consists of two terms: the first term is the computational error, and the second is the statistical error. The computational error decays exponentially with the number of iterations, whereas the statistical error remains the same as $t$ grows. The statistical error is unavoidable, as it reflects the statistical error due to estimation with noise; see also Theorem~\ref{thm:minimax}. For tensors with  $d_1=\cdots=d_K=d$, the computational error is dominated by the statistical error when the iteration number satisfies
\[
t\geq T=\log_{1/\rho}\left( {C_1 \text{Loss}(\Theta(\mA^{T_0}), \trueT) \over {C_2 L_{\alpha}\over \gamma_\alpha}\sqrt{R^{K-1}\sum_k d_k\over \prod_k d_k} } \right)+T_0 \asymp \log_{1/\rho}\left\{d^{(k-1)/2}\right\}.
\]

\subsection{Missing Data, Rank Selection, and Computational Complexity}\label{sec:miss-rank}

When some tensor entries $y_{i_1,\ldots,i_K}$ are missing, we replace the objective function $\tL_{\tY}(\Theta)$ with $\sum_{(i_1,\ldots,i_K)\in\Omega}\log f(q_{i_1,\ldots,i_K}\theta_{i_1,\ldots,i_K})$, where $\Omega\subset[d_1]\times\cdots\times[d_K]$ is the index set for non-missing entries. The same strategy has been used for continuous-valued tensor decomposition~\citep{acar2010scalable}. For implementation, we modify line 5 in Algorithm~\ref{alg:binary}, by fitting GLMs to the data for which $y_{i_1,\ldots,i_K}$ are observed. Other steps in Algorithm~\ref{alg:binary} are amendable to missing data accordingly. Our approach requires that there are no completely missing subtensors $\tY(:,j(k),:)$, which is a fairly mild condition. This requirement is similar to the coherence condition in the matrix completion problem; for instance, the recovery of true decomposition is impossible if an entire row or column of a matrix is missing. 

As a by-product, our tensor decomposition output can also be used for missing value prediction. That is, we predict the missing values $y_{i_1,\ldots,i_K}$ using $f(\hat \theta_{i_1,\ldots,i_K})$, where $\hat \Theta$ is the coefficient tensor estimated from the observed entries. Note that the predicted values are always between 0 and 1, which can be interpreted as a prediction for $\mathbb{P}(Y_{i_1,\ldots,i_K}=1)$. For accuracy guarantees with missing data, we refer to~\cite{lee2020tensor} for detailed results. 

Algorithm~\ref{alg:binary} takes the rank of $\Theta$ as an input. Estimating an appropriate rank given the data is of practical importance. We adopt the usual Bayesian information criterion (BIC) and choose the rank that minimizes BIC; i.e., 
\begin{equation}
\hat R=\arg\min_{R\in\mathbb{R}_{+}} \textrm{BIC}(R) = \arg\min_{R\in\mathbb{R}_{+}}\left[ -2\tL_\tY\{\hat\Theta(R)\} + p_e(R) \log\left(\prod_{k}d_k\right)\right],
\end{equation}
where $\hat\Theta(R)$ is the estimated coefficient tensor $\hat \Theta$ under the working rank $R$, and $p_e(R)$ is the effective number of parameters. This criterion aims to balance between the goodness-of-fit for the data and the degree of freedom in the population model. The empirical performance of BIC is investigated in Section~\ref{sec:empirical}. 

Finally, the computational complexity of our algorithm is $\tO(R^3 \prod_k d_k)$ for each iteration. The per-iteration computational cost scales linearly with the tensor dimension, and this complexity matches with the classical continuous-valued tensor decomposition~\citep{kolda2009tensor}. More precisely, the update of $\mA_k$ involves solving $d_k$ separate GLMs. Solving these GLMs requires $\tO(R^3 d_k + R^2 \prod_k d_k)$, and therefore the cost for updating $K$ factors in total is $\tO(R^3 \sum_k d_k + R^2 K \prod_k d_k)$. We further report the computation time in Section~\ref{sec:empirical}.

\section{Simulations}
\label{sec:simulations}\label{sec:empirical}

\subsection{CP Tensor Model}\label{sec:sim-cp}

In this section, we first investigate the finite-sample performance of our method when the data indeed follows the CP tensor model. We consider an order-3 dimension-$(d,d,d)$ binary tensor $\tY$ generated from the threshold model~\eqref{latent}, where $\trueT=\sum_{r=1}^R\ma^{(1)}_r\otimes \ma^{(2)}_r\otimes \ma^{(3)}_r$, and the entries of $\ma_r^{(k)}$ are i.i.d.\ drawn from Uniform$[-1,1]$ for all $k\in[3]$ and $r\in[R]$. Without loss of generality, we scale $\trueT$ such that $\mnormSize{}{\trueT}=1$. The binary tensor $\tY$ is generated based on the  entrywise quantization of the latent tensor $(\Theta^{\text{true}}+\tE)$, where $\tE$ consists of i.i.d.\ Gaussian entries. We vary the rank $R\in\{1,3,5\}$, the tensor dimension $d\in\{20,30,\ldots,60\}$, and the noise level $\sigma\in \{10^{-3},10^{-2.5},\ldots,10^{0.5}\}$. We use BIC to select the rank and report the estimation error based on logistic link averaged across $n_\text{sim}=30$ replications.

Figure~\ref{fig:SNR}(a) plots the estimation error $\text{Loss}(\trueT, \MLET)$ as a function of the tensor dimension $d$ while holding the noise level fixed at $\sigma = 10^{-0.5}$ for three different ranks $R\in\{1,3,5\}$. We find that the estimation error of the constrained MLE decreases as the dimension increases. Consistent with our theoretical results, the decay in the error appears to behave on the order of $d^{-1}$. A higher-rank tensor tends to yield a larger recovery error, as reflected by the upward shift of the curves as $R$ increases. Indeed, a higher rank means a higher intrinsic dimension of the problem, thus increasing the difficulty of the estimation.

\begin{figure}[t]
\centerline{\includegraphics[width=0.8\textwidth]{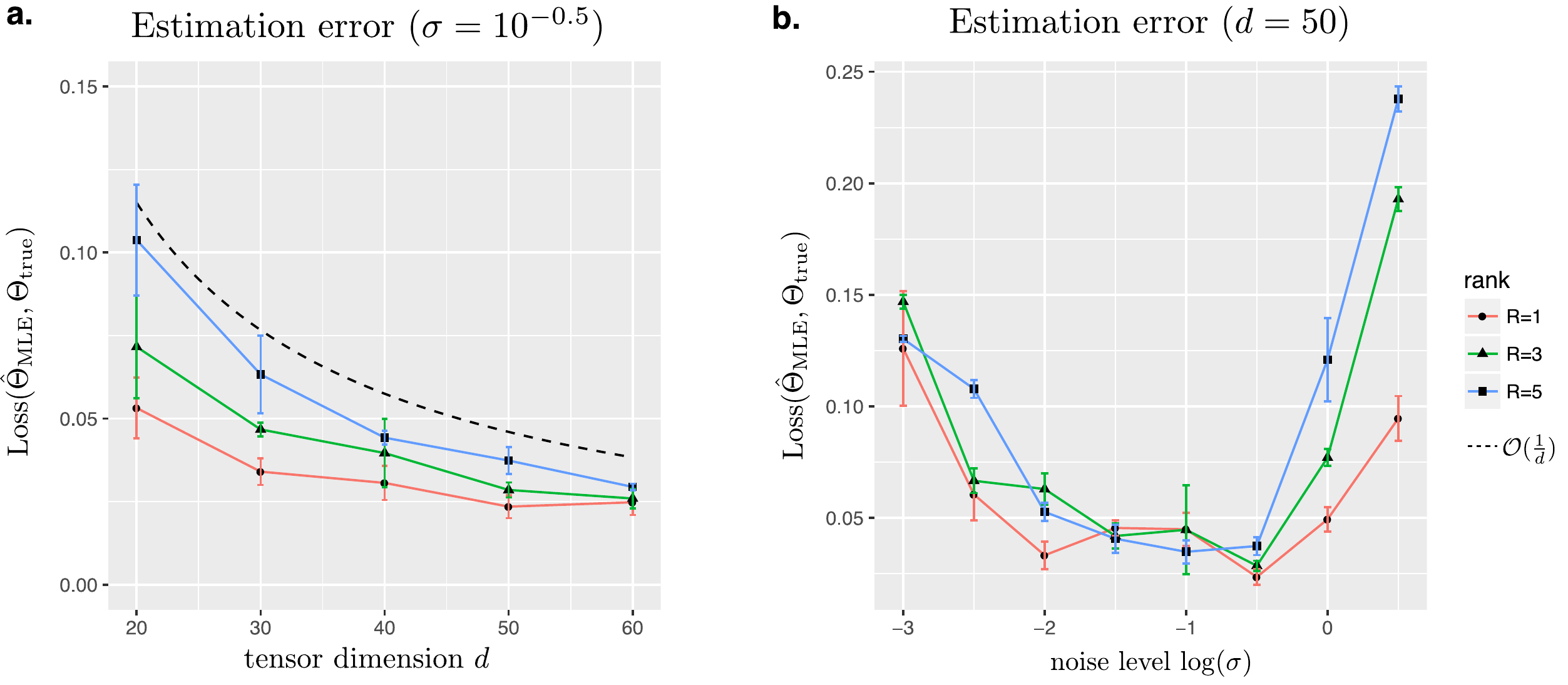}}
\caption{Estimation error of binary tensor decomposition. (a) Estimation error as a function of the tensor dimension $d=d_1=d_2=d_3$. (b) Estimation error as a function of the noise level.}
\label{fig:SNR}
\vspace{-.5cm}
\end{figure}

Figure~\ref{fig:SNR}(b) plots the estimation error as a function of the noise level $\sigma$ while holding the dimension fixed at $d = 50$ for three different ranks $R\in\{1,3,5\}$. A larger estimation error is observed when the noise is either too small or too large. The non-monotonic behavior confirms the phase transition with respect to the SNR. Particularly, the random noise is seen to improve the recovery accuracy in the high SNR regime. This is consistent to our theoretical result on the ``dithering'' effects brought by stochastic noise.

We next assess the tensor rank selection by BIC. We consider the tensor dimension $d\in\{20,40,60\}$ and rank $R\in \{5,10,20,40\}$. Note that, in some of the combinations, the rank equals or exceeds the tensor dimension. We set the noise level $\sigma\in\{0.1,0.01\}$ such that the noise is neither negligible nor overwhelming. For each combination, we simulate the tensor data following the Bernoulli tensor model~\eqref{eq:model}. We minimize BIC using a grid search from $R-5$ to $R+5$. Table~\ref{tab:rank} reports the selected rank averaged over $n_\text{sim}=30$ replications, with the standard error shown in the parenthesis. We find that, when $d=20$, the selected rank is slightly smaller than the true rank, whereas for $d\geq 40$, the selection is accurate. This agrees with our expectation, as the total number of entries corresponds to the sample size in tensor decomposition. A larger $d$ implies a larger sample size, so the BIC selection becomes more accurate.

\begin{table}[t!]
\centering
\begin{tabular}{c|ccc|ccc} \hline
 & \multicolumn{3}{c|}{$\sigma = 0.1$} &  \multicolumn{3}{c}{$\sigma = 0.01$} \\ \cline{2-7}
True rank & $d=20$ &  $d=40$ & $d=60$ & $d=20$ & $d=40$ & $d=60$ \\ \hline
$R=5$   & 4.9 (0.2) & 5 (0) & 5 (0) & 4.8 (1.0) & 5 (0) & 5 (0) \\
$R=10$ & 8.7 (0.9) & 10 (0) & 10 (0) & 8.8 (0.4) & 10 (0) & 10 (0) \\ 
$R=20$ & 17.7(1.7)&20.4(0.5)&20.2(0.5)&16.4(0.5)&20.4(0.5)&20.6(0.5) \\
$R=40$ &36.8(1.1)&39.6(1.7)& 40.2(0.4)&36.0(1.2)&38.8(1.6)&40.3(1.1) \\ \hline
\end{tabular}
\caption{Rank selection in binary tensor decomposition via BIC. The selected rank is averaged across $30$ simulations, with the standard error shown in the parenthesis.}
\label{tab:rank}
\end{table}

We also evaluate the numerical stability of our optimization algorithm. Although Algorithm~\ref{alg:binary} has no theoretical guarantee to land at the global optimum, in practice, we often find that the convergence point $\hat\Theta$ is satisfactory, in that the corresponding objective value $\tL_{\tY}(\hat\Theta)$ is close to and actually slightly larger than the objective function evaluated at the true parameter $\tL_{\tY}(\trueT)$. As an illustration, Figure~\ref{fig:traj} shows the typical trajectories of the objective function under different tensor dimensions and ranks. The dashed line is the objective value at the true parameter, $\tL_{\tY}(\trueT)$. We find that, upon random initializations, the algorithm lands at a good convergence point and converges quickly. It usually takes fewer than 8 iterations for the relative change in the objective to be below 3\%, even for a large $d$ and $R$. The average computation time per iteration is shown in the plot legend. For instance, when $d = 60$ and $R=10$, each iteration of Algorithm~\ref{alg:binary} takes fewer than 3 seconds on average.

\begin{figure}[t]
\centerline{\includegraphics[width=.7\textwidth]{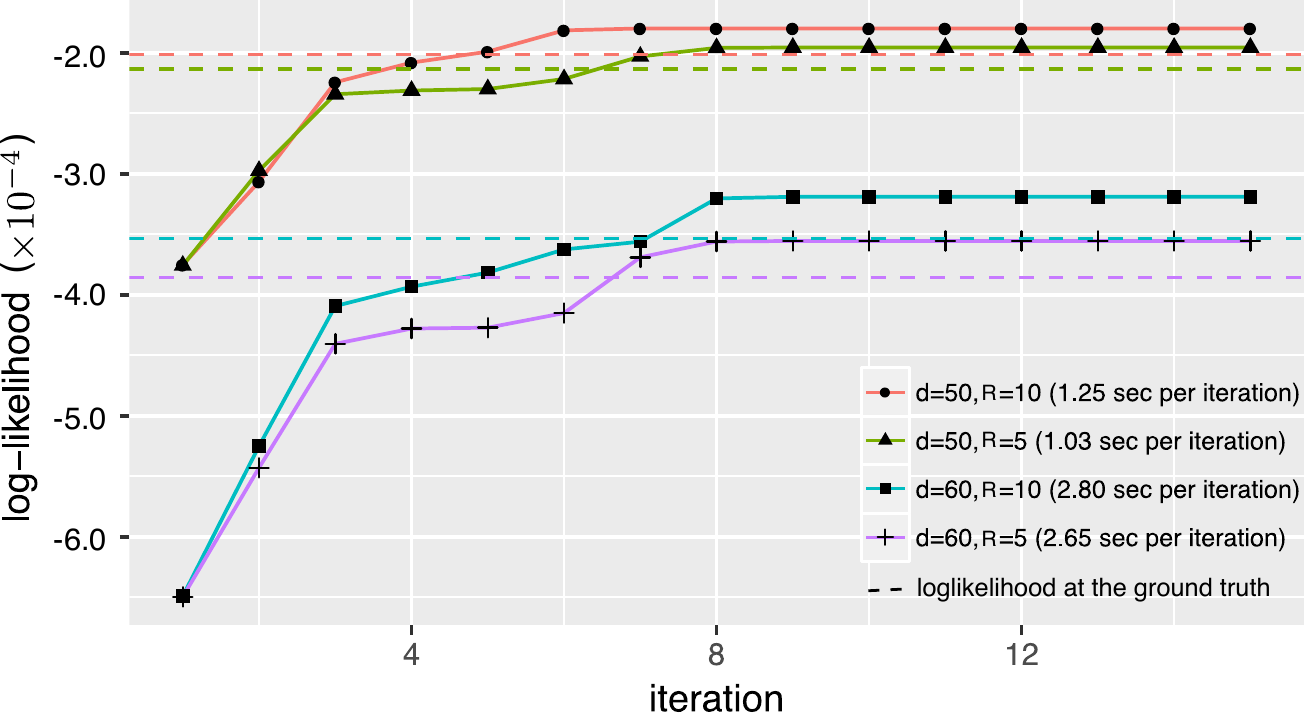}}
\caption{Trajectory of the objective function over iterations with varying $d$ and $R$.}
\label{fig:traj}
\vspace{-.5cm}
\end{figure}

\subsection{Stochastic Multi-way Block Model}\label{sec:block}

We next evaluate our method under the stochastic multi-way block model, which can be viewed as a higher-order generalization of the stochastic block model commonly used for random graphs, network analysis, and community detection. Under this model, the signal tensor does not have an explicit CP structure with known rank. Specifically, we generate $\tY$ of dimension $d=d_1=d_2=d_3$, where we vary $d\in\{20,30,40,50,60\}$. The entries in $\tY$ are realizations of independent Bernoulli variables with a probability tensor $\Theta$. The probability tensor $\Theta$ has five blocks along each of the modes, 
\begin{equation}
\text{Probit}^{-1}(\Theta)=\tC\times_1\mN_1\times_2\mN_2\times_3\mN_3,
\end{equation}
where $\mN_1, \mN_2, \mN_3 \in \{0,1\}^{d\times 5}$ are membership matrices indicating the block allocation along each of the mode, $\times_k$ denotes the tensor-by-matrix multiplication~\citep{kolda2009tensor} for $k\in[3]$, and $\tC=\entry{c_{m_1m_2m_3}}\in\mathbb{R}^{5\times 5\times 5}$ is a core tensor corresponding to the block-means on a probit scale,  and $m_1,m_2,m_3 \in \{1,\ldots,5\}$ are block indices. We generate the block means $c_{m_1m_2m_3}$ in the following ways:
\begin{itemize}
\item Combinatorial-mean model: $c_{m_1m_2m_3}\stackrel{\text{i.i.d.}}{\sim} \text{Uniform}[-1,1]$; i.e., each three-way block has its own mean, independent of each other. 
\item Additive-mean model: $c_{m_1m_2m_3}=c_{m_1}^1+\mu_{m_2}^2+\mu_{m_3}^3$, where $\mu^1_{m_1}$, $\mu^2_{m_2}$ and $\mu^3_{m_3}$ are i.i.d.\ drawn from $\text{Unif}[-1,1]$. 
\item Multiplicative-mean model: $c_{m_1m_2m_3} = c_{m_1}^1\mu_{m_2}^2\mu_{m_3}^3$, and the rest of setup is the same as the additive-mean model.
\end{itemize}

We evaluate our method in terms of the accuracy of recovering the latent tensor $\Theta$ given the binary observations. Table~\ref{tab:block} reports the relative loss, the estimated rank, and the running time, averaged over $n_\text{sim}=30$ data replications, for the above three sub-models. The relative loss is computed as $\FnormSize{}{\MLET-\trueT}/\FnormSize{}{\trueT}$. Our method is able to recover the signal tensors well in all three scenarios. As an illustration, we also plot one typical realization of the true signal tensor, the input binary tensor, and the recovered signal tensor for each sub-model in Table~\ref{tab:block}. It is interesting to see that, not only the block structure but also the tensor magnitude are well recovered. We remark that, the data has been generated from a probit model, but we always fit with a logistic link. Our method is shown to maintain a reasonable performance under this model misspecification.

\begin{table}[t!]
\resizebox{\columnwidth}{!}{
\begin{tabular}{c|c|c|c|c}\hline
\multirow{2}{*}{Block model}&Experiment& Relative & Rank& Time\\
\cline{2-2}
&True signal  $\quad$ Input tensor $\quad$ Output tensor& Loss &Estimate& (sec)\\
\hline
Additive&
\begin{minipage}{.45\textwidth} 
\includegraphics[width=7cm]{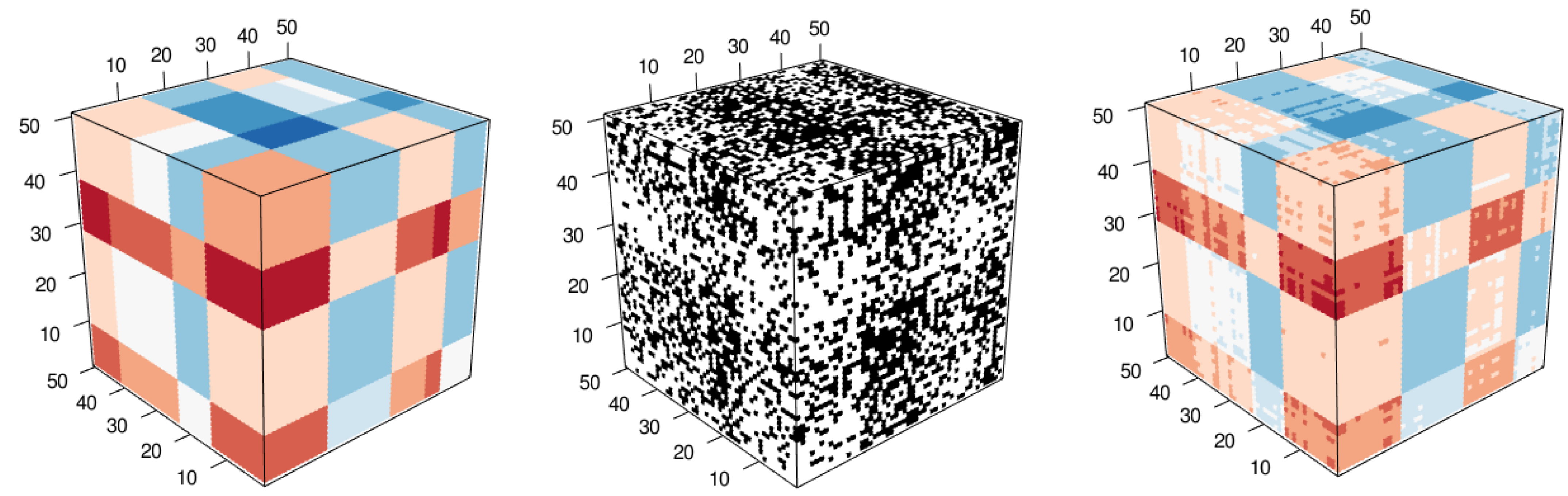}
\end{minipage}
 &0.23(0.05)&1.9(0.3)& 4.23(1.62)\\  
 
\hline
Multiplicative&
\begin{minipage}{.45\textwidth} 
\includegraphics[width=7cm]{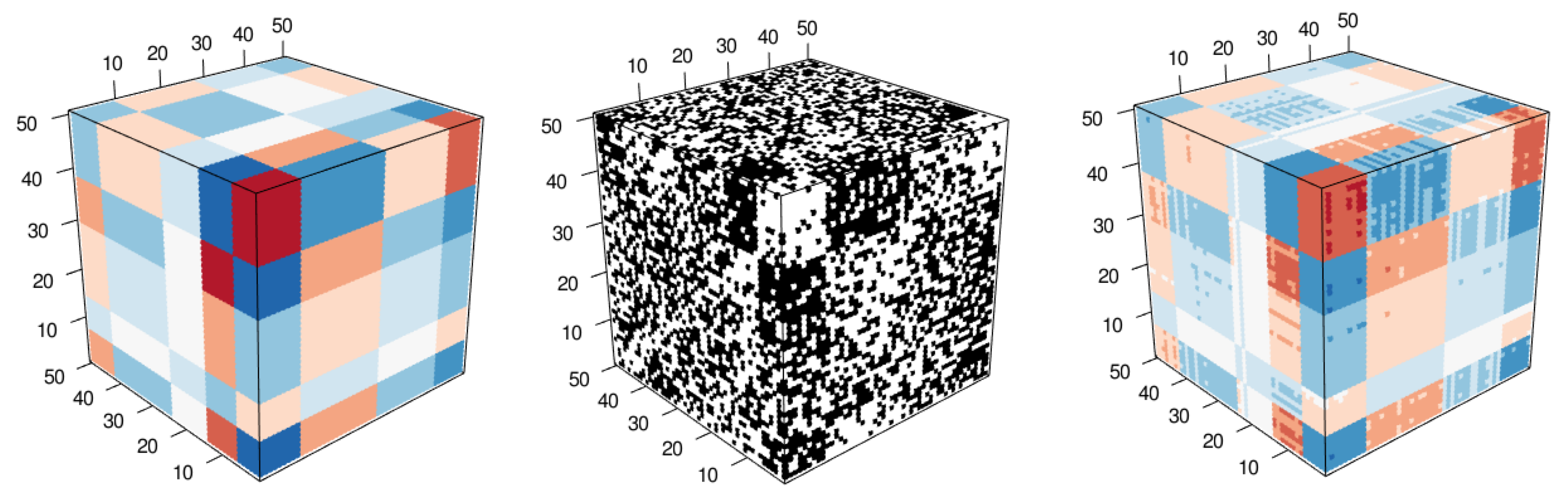}
 \end{minipage}
 &0.22(0.07)&1.0(0.0)&1.70(0.09)\\  
 
 \hline
Combinatorial&
\begin{minipage}{.45\textwidth} 
\includegraphics[width=7cm]{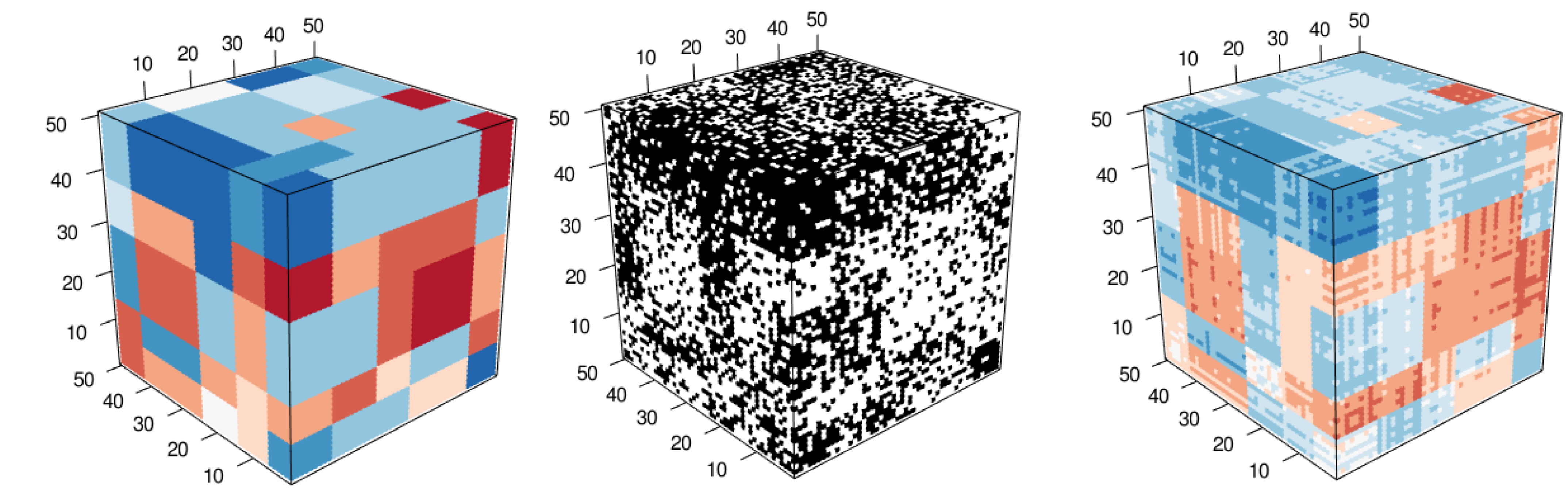}
\end{minipage}
 &0.48(0.04)&6.0(0.9)&10.4(3.4)\\  
 \hline
  
\end{tabular}
}

\hspace{2.5cm}
\begin{minipage}{.1\textwidth} 
\includegraphics[width=5cm]{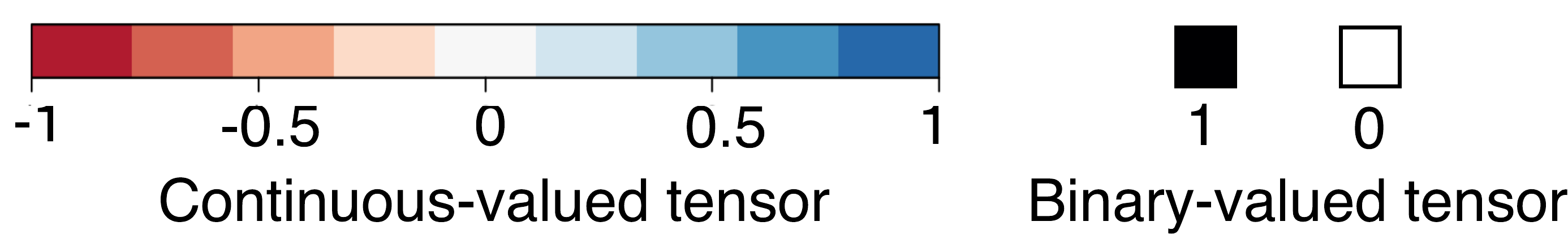}
 \end{minipage}
 
\caption{Latent tensor recovery. Figures in the column of ``Experiment'' are color images of the simulated tensor under different block mean models. Reported are the relative loss, estimated rank, and running time, averaged over 30 data replications. Standard error is shown in the parenthesis.}\label{tab:block}
\vspace{-.5cm}
\end{table}

\subsection{Comparison with Alternative Methods}\label{sec:sensitivity}

We next compare our method with a number of alternative solutions for binary tensor decomposition.  
\begin{itemize}
\item Boolean tensor factorization (BooleanTF)~\citep{miettinen2011boolean,erdos2013walk,rukat2018probabilistic}. This method decomposes a binary tensor into binary factors and then recovers the binary entries based on a set of logical rules among the factors. We use the implementation of \cite{rukat2018probabilistic}.  

\item Bayesian tensor factorization (BTF\_Bayeisan)~\citep{rai2014scalable}. This method uses expectation-maximization to decompose a binary tensor into continuous-valued factors. The algorithm imposes a Gaussian prior on the factor entries and a multiplicative gamma process prior on the factor weights $\{\lambda_r\}$.

\item Bernoulli tensor factorization with gradient descent (BTF\_Gradient)~\citep{hong2020generalized}. This method uses a gradient descent algorithm to decompose a binary tensor into continuous-valued factors. We use the implementation in the toolbox of Matlab. 
\end{itemize}

For easy reference, we denote our method by BTF\_Alternating\footnote{Software implementation: \url{https://github.com/Miaoyanwang/Binary-Tensor}}. These four methods differ in several ways. BooleanTF is different from the other three in both the cost function and the output format. The rest are all based on the Bernoulli model \eqref{eq:model}, but with different implementations. BTF\_Bayesian employs a Bayesian approach, whereas the other two are frequentist solutions. BTF\_Gradient and our method, BTF\_Alternating, share the same model, but utilize different optimization algorithms. So the two methods complement each other. On the other hand, we provide not only the algorithm-specific convergence properties, but also algorithm-independent statistical properties including the statistical convergence rate, SNR phase diagram, and mini-max rate. These results are not available in the proposal of BTF\_Gradient~\citep{hong2020generalized}. 

We apply the four methods with default parameters, while selecting the rank $R$ using the recommended approach of each. For our method BTF\_Alternating, we use the proposed BIC to select the rank. Because BTF\_Gradient does not provide any rank selection criterion, we apply the same $R$ selected by our BIC. For BTF\_Alternating, we set the hyper-parameter $\alpha$ to infinity, which essentially poses no prior on the tensor magnitude. Besides, because BTF\_Bayesian only supports the logistic link, we use the logistic link in all three BTF methods. 

We evaluate each method by two metrics. The first metric is the root mean square error, $\text{RMSE}=\left(\sqrt{\prod_k d_k}\right)^{-1} \FnormSize{}{\widehat{ \mathbb{E}(\tY)}-\mathbb{E}(\tY)}$, where $\widehat{ \mathbb{E}(\tY)}$ denotes the estimated probability tensor. For BooleanTF, this quantity is represented as the posterior mean of $\tY$~\citep{miettinen2011boolean}, and for the other three methods, $\widehat{ \mathbb{E}(\tY)}=\text{logit}(\hat{\Theta})$. The second metic is the misclassification error rate, $\text{MER}=\left(\prod_k d_k\right)^{-1} \zeronormSize{}{\mathds{1}_{\widehat{\mathbb{E}(\tY)}\geq 0.5}-\mathds{1}_{\mathbb{E}(\tY)\geq 0.5}}$. Here the indicator function is applied to tensors in an element-wise manner, and $\zeronormSize{}{\cdot}$ counts the number of non-zero entries in the tensor. These metrics reflect two aspects of the statistical error. RMSE summarizes the estimation error in the parameters, whereas MER summarizes the classification errors among 0's and 1's. 

We simulate data from two different models, and in both cases, the signal tensors do not necessarily follow an exact low-rank CP structure. Therefore, in addition to method comparison, it also allows us to evaluate the robustness of our method under potential model misspecification. 

The first model is a boolean (logical) tensor model following the setup in \cite{rukat2018probabilistic}. We first simulate noiseless tensors $\tY=\entry{y_{ijk}}$ from the following model,
\begin{equation}
\displaystyle y_{ijk} = \bigvee _{r=1}^R \bigwedge_{ijk} a_{ir} b_{jr} c_{kr}, \;
\textrm{with}\ a_{ir} \sim \text{Ber}(p^{a}_{ir}),\ b_{jr}\sim \textrm{Bernoulli}(p^{b}_{jr}), c_{kr}\sim \textrm{Bernoulli}(p^{c}_{kr}),
\end{equation}
where the binary factor entries $\{a_{ir}\}$, $\{b_{jr}\}$, $\{c_{kr}\}$ are mutually independent with each other, the factor probabilities $\{p^a_{ir}\}$, $\{p^b_{jr}\}$, $\{p^c_{kr}\}$ are generated i.i.d.\ from Beta(2,4), and $\vee$ and $\wedge$ denote the logical OR and AND operations, respectively. Equivalently, the tensor entry is 1 if and only if there exists one or more components in which all corresponding factor entries are 1. It is easy to verify that
\[
\mathbb{E}(y_{ijk}|\{p^a_{ir},p^b_{jr},p^c_{kr}\})=1-\prod_{r=1}^R\left(1-p^{a}_{ir}p^{b}_{jr}p^{c}_{kr}\right).
\]
We then add contamination noise to $\tY$ by flipping the tensor entries $0\leftrightarrow 1$ i.i.d.\ with probability 0.1. We consider the tensor dimension $d_1=d_2=d_3=50$ and the boolean rank $R\in\{10,15,20,25,30\}$.

Figure~\ref{fig:simulation}(a)-(b) shows the performance comparison based on $n_{\text{sim}}=30$ replications. We find that the three BTF methods outperform BooleanTF in RMSE. The results shows the advantage of a probabilistic model, upon which all three BTF methods are built. In contrast, BooleanTF seeks patterns in a specific data realization, but does not target for population estimation. For classification, BooleanTF performs reasonably well in distinguishing 0's versus 1's, which agrees with the data mining nature of BooleanTF. It is also interesting to see that MER peaks at $R= 20$. Further investigation reveals that this setting corresponds to the case when the Bernoulli probabilities $\mathbb{E}(\tY)$ concentrate around $0.5$, which becomes particularly challenging for classification. Actually, the average Bernoulli probability for $R=$10, 15, 20, 25, 30 is  0.31, 0.44, 0.53, 0.61, 0.68, respectively. Figure~\ref{fig:simulation}(b) also shows that BTF\_Alternating and BTF\_Gradient achieve a smaller  classification error than BTF\_Bayesian. One possible explanation is that the normal prior in BTF\_Bayesian has a poor distinguishing power around $\theta \approx 0$, which corresponds to the hardest case when Bernoulli probability $\approx 0.5$.

\begin{figure}[t]
\begin{center}
\includegraphics[width=.9\textwidth]{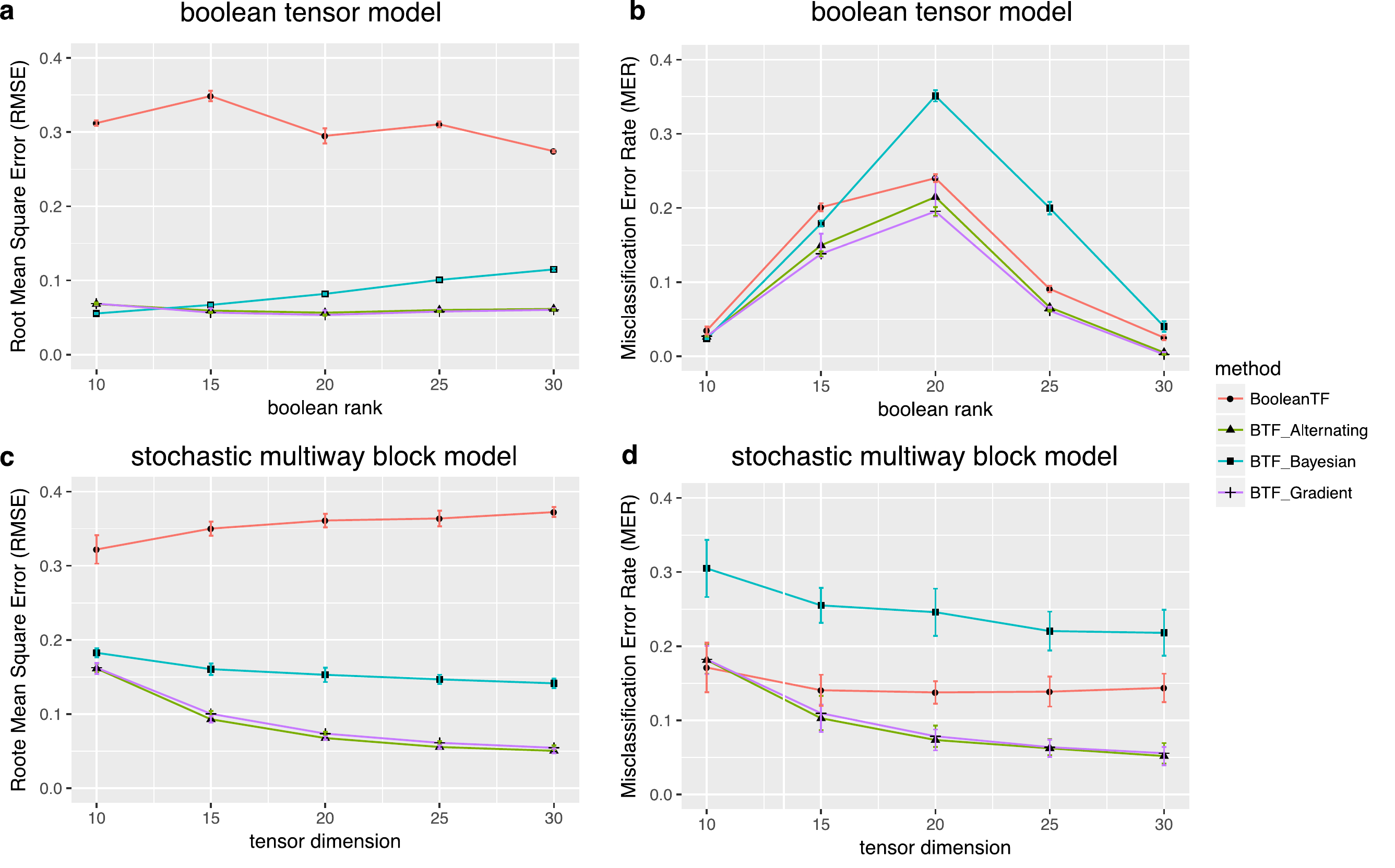}
\caption{Performance comparison in terms of root mean squared error and misclassification error rate. (a)-(b) Estimation errors under the boolean tensor model. (c)-(d) Estimation errors under the stochastic multiway block model. Error bars represent one standard error around the mean.}\label{fig:simulation}
\vspace{-1cm}
\end{center}
\end{figure}

The second model is the stochastic multi-way block model considered in Section~\ref{sec:block}, with the block means $\{c_{m_1m_2m_3}\}$ generated from the combinatorial-mean sub-model. Figure~\ref{fig:simulation}(c)-(d) shows the performance comparison, and a similar pattern is observed. The two frequentist-type BTF methods, BTF\_Gradient and BTF\_Alternating, behave numerically similarly, and they outperform the other alternatives. In particular, the BTF methods exhibit decaying estimation errors, whereas BooleanTF appears to flatten out as dimension grows. This observation suggests that, compared to the algorithmic error, the statistical error is likely more dominating in this setting.

\section{Data Applications}
\label{sec:realdata}

We next illustrate the applicability of our binary tensor decomposition method on a number of data sets, with applications ranging from social networks, email communication networks, to brain structural connectivities. We consider two tasks: one is tensor completion, and the other is clustering along one of the tensor modes. The data sets include: 
\begin{itemize}
\item \textit{Kinship}~\citep{nickel2011three}: This is a 104 $\times$ 104 $\times$ 26 binary tensor consisting of 26 types of relations among a set of 104 individuals in Australian Alyawarra tribe. The data was first collected by~\cite{denham2005multiple} to study the kinship system in the Alyawarra language. The tensor entry $\tY(i,j,k)$ is 1 if individual $i$ used the kinship term $k$ to refer to individual $j$, and 0 otherwise.  

\item \textit{Nations}~\citep{nickel2011three}: This is a 14 $\times$ 14 $\times$ 56 binary tensor consisting of 56 political relations of 14 countries between 1950 and 1965. The tensor entry indicates the presence or absence of a political action, such as ``treaties'', ``sends tourists to'', between the nations. We note that the relationship between a nation and itself is not well defined, so we exclude the diagonal elements $\tY(i,i,k)$ from the analysis. 

\item \textit{Enron}~\citep{zhe2016distributed}: This is a 581 $\times$ 124 $\times$ 48 binary tensor consisting of the three-way relationship, (sender, receiver, time), from the Enron email data set. The Enron data is a large collection of emails from Enron employees that covers a period of 3.5 years. Following~\cite{zhe2016distributed}, we take a subset of the Enron data and organize it into a binary tensor, with entry $\tY(i,j,k)$ indicating the presence of emails from a sender $i$ to a receiver $j$ at a time period $k$. 

\item \textit{HCP}~\citep{wang2019common}: This is a 68 $\times$ 68 $\times$ 212 binary tensor consisting of structural connectivity patterns among 68 brain regions for 212 individuals from Human Connectome Project (HCP). All the individual images were preprocessed following a standard pipeline~\citep{zhang2018mapping}, and the brain was parcellated to 68 regions-of-interest following the Desikan atlas~\citep{desikan2006automated}. The tensor entries encode the presence or absence of fiber connections between those 68 brain regions for each of the 212 individuals. 
\end{itemize}

The first task is binary tensor completion, where we apply tensor decomposition to predict the missing entries in the tensor. We compare our binary tensor decomposition method using a logistic link function with the classical continuous-valued tensor decomposition. Specifically, we split the tensor entries into 80\% training set and 20\% testing set, while ensuring that the nonzero entries are split the same way between the training and testing data. The entries in the testing data are masked as missing, and we predict them based on the tensor decomposition from the training data. The training-testing split is repeated five times, and we report the average area under the receiver operating characteristic curve (AUC) and RMSE across five splits in Table~\ref{tab:link}. It is clearly seen that the binary tensor decomposition substantially outperforms the classical continuous-valued tensor decomposition. In all data sets, the former obtains a much higher AUC and mostly a lower RMSE. We also report in Table~\ref{tab:link} the percentage of nonzero entries for each data. We find that our decomposition method performs well even in the sparse setting. For instance, for the Enron data set, only 0.01\% of the entries are non-zero. The classical decomposition almost blindly  assigns 0 to all the hold-out testing entires, resulting in a poor AUC of 79.6\%. By comparison, our binary tensor decomposition achieves a much higher classification accuracy, with AUC = 94.3\%.

\begin{table}[t!]
\centering
\begin{tabular}{lc|cc|cc} \hline
&  &\multicolumn{4}{c}{Tensor decomposition method}\\ \cline{3-6}
Data set& Non-zeros&\multicolumn{2}{c|}{Binary (logistic link)} &  \multicolumn{2}{c}{Continuous-valued }\\\hline
&&AUC& RMSE& AUC & RMSE\\
\emph{Kinship}&3.80\%&0.9708&$1.2\times 10^{-4}$& 0.9436&$1.4\times 10^{-3}$\\
\emph{Nations}& 21.1\%&0.9169&$1.1\times 10^{-2}$&0.8619&$2.2\times 10^{-2}$\\
\emph{Enron}&0.01\%&0.9432&$6.4\times 10^{-3}$& 0.7956&$6.3\times 10^{-5}$\\ 
\emph{HCP}&35.3\% &0.9860& $1.3\times 10^{-3}$&0.9314&$1.4\times 10^{-2}$\\ \hline
\end{tabular}
\caption{Tensor completion for the four binary tensor data sets using two methods: the proposed binary tensor decomposition, and the classical continuous-valued tensor decomposition.}
\label{tab:link}
\end{table}

\begin{figure}[t]
\centering
\includegraphics[width=.9\textwidth]{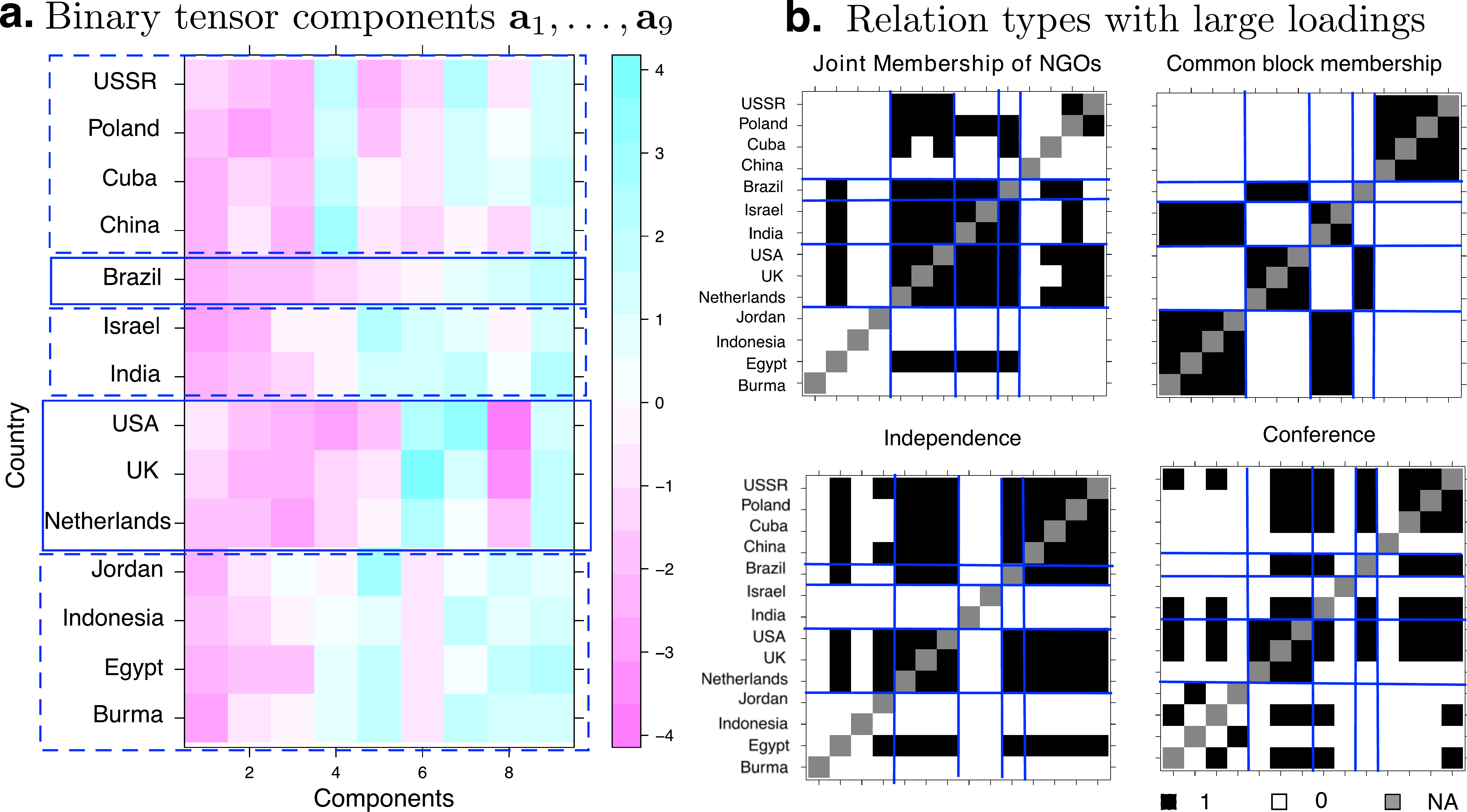}
\caption{Analysis of the \emph{Nations} data set. (a) Top nine tensor components in the country mode from the binary tensor decomposition. The overlaid box depicts the results from the $K$-means clustering. (b) Relation types with large loadings. Top four relationships identified from the top tensor components are plotted. }\label{fig:nation}
\vspace{-.5cm}
\end{figure}

The second task is clustering. We perform the clustering analyses on two data sets, \emph{Nations} and \emph{HCP}. For the \emph{Nations} data set, we utilize a two-step procedure by first applying the proposed binary tensor decomposition method with the logistic link, then applying the $K$-means clustering along the country mode from the decomposition. In the first step, the BIC criterion suggests $R = 9$ factors, and in the second step, the classical elbow method selects 5 clusters out of the 9 components. Figure~\ref{fig:nation}(a) plots the 9 tensor factors along the country mode. It is interesting to observe that the countries are partitioned into one group containing those from the communist bloc, two groups from the western bloc, two groups from the neutral bloc, and Brazil forming its own group. We also plot the top four relation types based on their loadings in the tensor factors along the relationship mode in Figure~\ref{fig:nation}(b). The partition of the countries is consistent with their relationship patterns in the adjacency matrices. Indeed, those countries belonging to the same group tend to have similar linking patterns with other countries, as reflected by the block structure in Figure~\ref{fig:nation}(b).

\begin{figure}[t]
\centering
\includegraphics[width=0.6\textwidth]{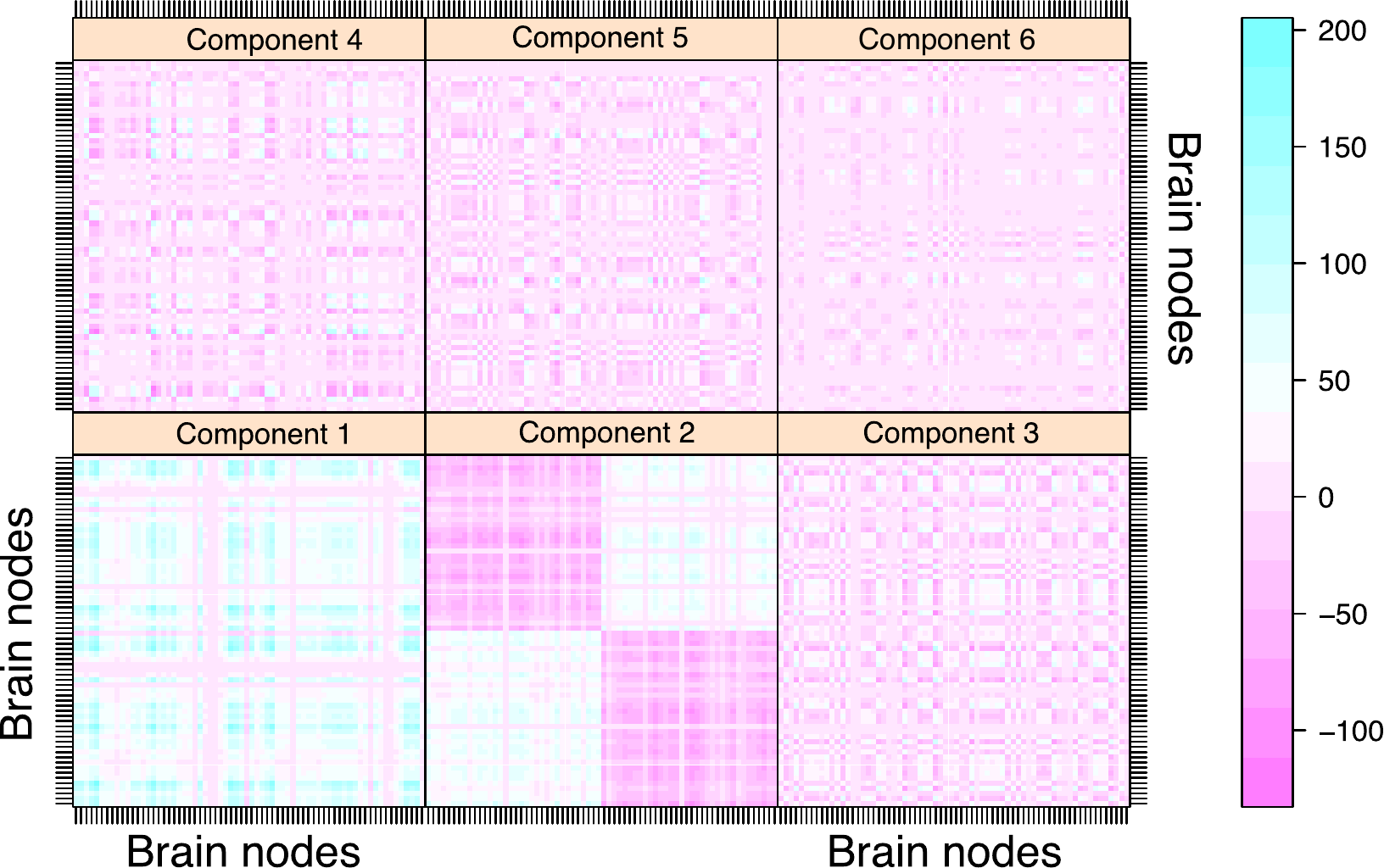}
\caption{Heatmap for binary tensor components across brain regions in the \emph{HCP} analysis. The connection matrix $\mA_r=\lambda_r\ma_r\otimes \ma_r$ is plotted for component $r\in[6]$.}
\label{fig:brain}
\end{figure}

\begin{figure}[t]
\centering
\includegraphics[width=0.6\textwidth]{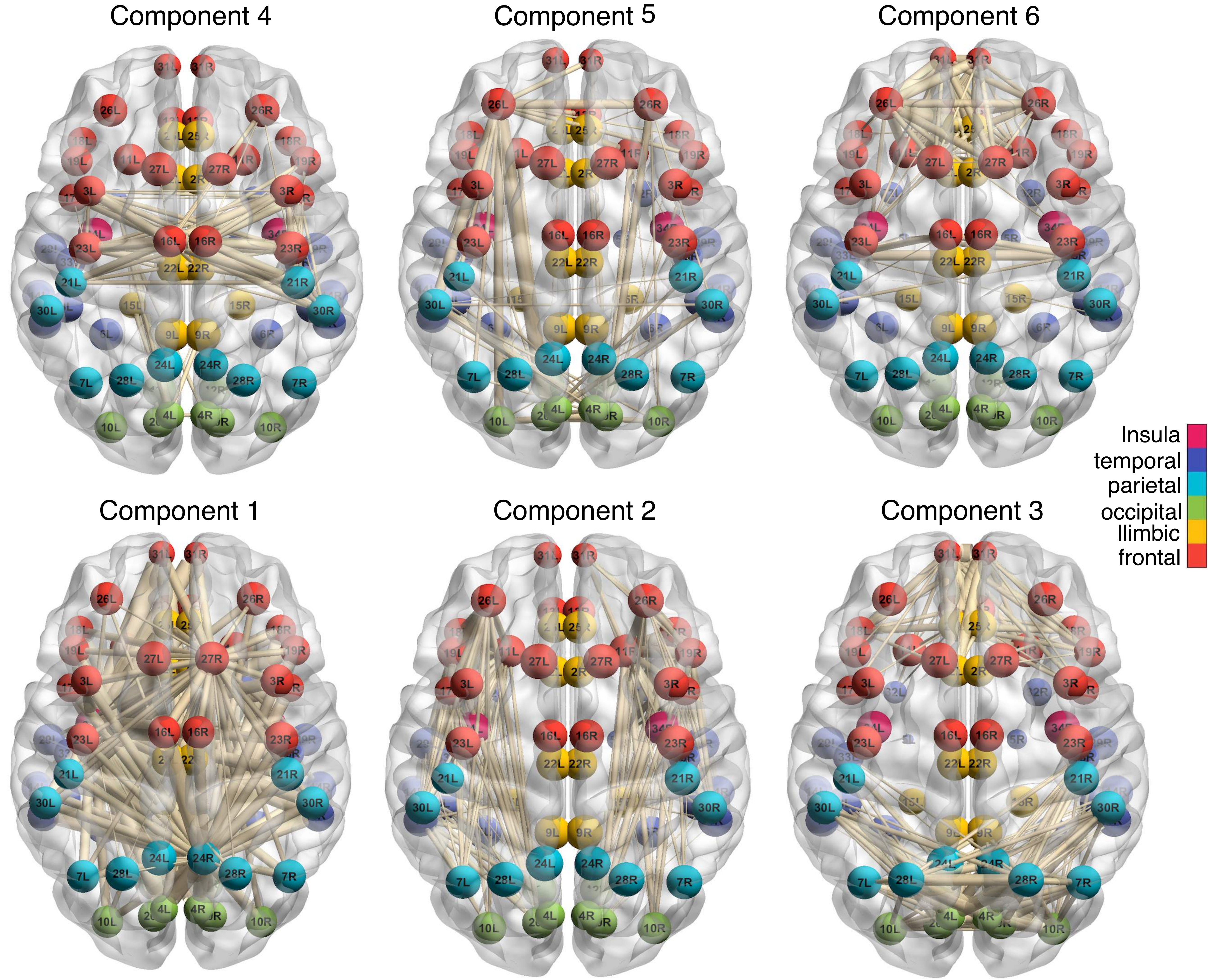}
\caption{Edges with high loadings in the \emph{HCP} analysis. The top 10\% edges with positive loadings $\mA_r(i,j)$ are plotted, for $r\in[6]$ and $(i,j)\in[68]^2$. The width of the edge is proportional to the magnitude of $\mA_r(i,j)$.}
\label{fig:component}
\vspace{-.5cm}
\end{figure}

We also perform the clustering analysis on the data set \emph{HCP}. We apply the decomposition method with the logistic link and BIC-selected rank $R = 6$. Figure~\ref{fig:brain} plots the heatmap for the top 6 tensor components across the 68 brain regions, and Figure~\ref{fig:component} shows the edges with high loadings based on the tensor components. Edges are overlaid on the brain template BrainMesh\_ICBM152~\citep{xia2013brainnet}, and nodes are color coded based on their regions. We see that the brain regions are spatially separated into several groups and that the nodes within each group are more densely connected with each other. Some interesting spatial patterns in the brain connectivity are observed. For instance, the edges captured by tensor component 2 are located within the cerebral hemisphere. The detected edges are association tracts consisting of the long association fibers, which connect different lobes of a hemisphere, and the short association fibers, which connect different gyri within a single lobe. In contrast, the edges captured by tensor component 3 are located across the two hemispheres. Among the nodes with high connection intensity, we identify superior frontal gyrus, which is known to be involved in self-awareness and sensory system~\citep{goldberg2006brain}. We also identify corpus callosum, which is the largest commissural tract in the brain that connects two hemispheres. This is consistent with brain anatomy that suggests the key role of corpus callosum in facilitating interhemispheric connectivity~\citep{roland2017role}. Moreover, the edges shown in tensor component 4 are mostly located within the frontal lobe, whereas the edges in component 5 connect the frontal lobe with parietal lobe.

\section{Proofs}\label{sec:proofs}

\subsection{Proof of Theorem~\ref{thm:rate}}
\label{sec:proof-thm1}

\begin{proof}
It follows from the expression of $\tL_\tY(\Theta)$ that
\begin{align}
{\partial \tL_\tY\over \partial \theta_{i_1,\ldots,i_K}}&={\dot{f}(\theta_{i_1,\ldots,i_K})\over f(\theta_{i_1,\ldots,i_K})}\mathds{1}_{\{y_{i_1,\ldots,i_K}=1\}}-{\dot{f}(\theta_{i_1,\ldots,i_K})\over 1-f(\theta_{i_1,\ldots,i_K})}\mathds{1}_{\{y_{i_1,\ldots,i_K}=-1\}},\\
{\partial^2 \tL_\tY\over\partial \theta^2_{i_1,\ldots,i_K}}&=-\left[  { \dot{f}^2(\theta_{i_1,\ldots,i_K}) \over f^2(\theta_{i_1,\ldots,i_K})} -{\ddot{f}(\theta_{i_1,\ldots,i_K}) \over f(\theta_{i_1,\ldots,i_K})} \right]\mathds{1}_{\{y_{i_1,\ldots,i_K}=1\}}\\
&\quad -\left[ { \ddot{f}(\theta_{i_1,\ldots,i_K}) \over 1-f(\theta_{i_1,\ldots,i_K})} +{\dot{f}^2(\theta_{i_1,\ldots,i_K}) \over 
\{1-f(\theta_{i_1,\ldots,i_K})\}^2} \right]\mathds{1}_{\{y_{i_1,\ldots,i_K}=-1\}}, \\
\partial^2 \tL_\tY\over\partial \theta_{i_1,\ldots,i_K}\partial \theta_{i'_1,\ldots,i'_K}&=0,\quad\text{if}\quad (i_1,\ldots,i_K)\neq (i'_1,\ldots,i'_K).
\end{align}
Define
\[
\tS_{\tY}(\trueT)=\left\llbracket \partial \tL_\tY\over \partial \theta_{i_1,\ldots,i_K} \right\rrbracket \Big|_{\Theta=\trueT}, \quad \textrm{ and } \quad
\tH_{\tY}(\trueT)=\left\llbracket \partial^2 \tL_\tY\over\partial \theta_{i_1,\ldots,i_K}\partial \theta_{i'_1,\ldots,i'_K} \right\rrbracket \Big|_{\Theta=\trueT},
\]
where $\tS_{\tY}(\trueT)$ is the collection of the score functions evaluated at $\trueT$, and $\tH_{\tY}(\trueT)$ is the collection of the Hession functions evaluated at $\trueT$. We organize the entries in $\tS_{\tY}(\trueT)$ and treat $\tS_{\tY}(\trueT)$ as an order-$K$ dimension-$(d_1,\ldots,d_K)$ tensor. Similarly, we organize the entries in $\tH_\tY(\trueT)$ and treat $\tH_\tY(\trueT)$ as a $\prod_k d_k$-by-$\prod_k d_k$ matrix. By the second-order Taylor's theorem, we expand $\tL_\tY(\Theta) $ around $\trueT$ and obtain
\begin{equation}\label{eq:taylor}
\tL_{\tY}(\Theta)=\tL_{\tY}(\trueT)+\langle S_{\tY}(\trueT), \Theta-\trueT  \rangle+{1\over 2} \Vec(\Theta-\trueT)^T \tH_{\tY}(\check\Theta)\Vec(\Theta-\trueT),
\end{equation}
where $\check\Theta = \gamma \trueT+(1-\gamma )\Theta$ for some $\gamma\in[0,1]$, and $\tH_{\tY}(\check\Theta)$ denotes the $\prod_k d_k$-by-$\prod_k d_k$ Hession matrix evaluated at $\check\Theta$. 

We first bound the linear term in \eqref{eq:taylor}. Note that, by Lemma~\ref{lem:inq}, 
\begin{equation}\label{eq:linear}
|\langle S_{\tY}(\trueT), \Theta-\trueT  \rangle|\leq \normSize{}{S_{\tY}(\trueT)} \nnormSize{}{\Theta-\trueT}.
\end{equation}
Define 
\[
s_{i_1,\ldots,i_K}={\partial \tL_\tY\over \theta_{i_1,\ldots,i_K}}\big|_{\Theta=\trueT} \;\; \textrm{ for all } \; (i_1,\ldots,i_K)\in[d_1]\times\cdots\times [d_K].
\]
It follows from model~\eqref{eq:model} and the expression for $L_\alpha$ that $S_{\tY}(\trueT)=\entry{s_{i_1,\ldots,i_K}}$ is a random tensor whose entries are independently distributed and satisfy 
\begin{equation}\label{eq:norm}
\mathbb{E}(s_{i_1,\ldots,i_K})=0,\quad |s_{i_1,\ldots,i_K}|\leq L_\alpha, \quad \text{for all }(i_1,\ldots,i_K)\in[d_1]\times \cdots \times [d_K].
\end{equation}
By Lemma~\ref{lem:noisytensor}, with probability at least $1-\exp(-C_1 \log K \sum_kd_k)$, we have
\begin{equation}\label{eq:normrandom} 
\normSize{}{S_\tY(\trueT)} \leq C_2 L_\alpha\sqrt{\sum_k d_k},
\end{equation}
where $C_1, C_2$ are two positive constants. Furthermore, note that $\text{rank}(\Theta)\leq R$, $\text{rank}(\trueT)\leq R$, so $\text{rank}(\Theta-\trueT)\leq 2R$. By Lemma~\ref{lem:nuclear}, $\nnormSize{}{\Theta-\trueT}\leq (2R)^{K-1\over 2}\FnormSize{}{\Theta-\trueT}$. Combining~\eqref{eq:linear}, \eqref{eq:norm} and \eqref{eq:normrandom}, we have that, with probability at least $1-\exp(-C_1 \log K \sum_kd_k)$,
\begin{equation}\label{eq:linearconclusion}
|\langle S_{\tY}(\trueT), \Theta-\trueT  \rangle | \leq C_2 L_\alpha  \sqrt{R^{K-1} \sum_k d_k}  \FnormSize{}{\Theta-\trueT},
\end{equation}
where the constant $C_2$ absorbs all factors that depend only on $K$. 

We next bound the quadratic term in \eqref{eq:taylor}. Notice that
\begin{align}\label{eq:quadratic}
 \Vec(\Theta-\trueT)^T H_{\tY}(\check{\Theta})\Vec(\Theta-\trueT)&=\sum_{i_1,\ldots,i_K}\left( {\partial^2\tL_{\tY}\over \partial \theta^2_{i_1,\ldots,i_K}} \Big|_{\Theta=\check\Theta} \right)(\Theta_{i_1,\ldots,i_K}-\Theta_{{\text{true}},i_1,\ldots,i_K})^2 \nonumber \\
&\leq - \gamma_\alpha\sum_{i_1,\ldots,i_K}(\Theta_{i_1,\ldots,i_K}-\Theta_{\text{true},i_1,\ldots,i_K})^2 \nonumber \\
&=-\gamma_\alpha\FnormSize{}{\Theta-\trueT}^2,
\end{align}
where the second line comes from the fact that  $\mnormSize{}{\check\Theta}\leq \alpha$ and the definition of $\gamma_\alpha$.

Combining~\eqref{eq:taylor}, \eqref{eq:linearconclusion} and~\eqref{eq:quadratic}, we have that, for all $\Theta\in\tD$, with probability at least $1-\exp(-C_1 \log K \sum_kd_k)$, 
\begin{equation}
\tL_\tY(\Theta)\leq \tL_{\tY}(\trueT)+C_2L_\alpha  \left(R^{K-1}\sum_k d_k\right)^{1/2}  \FnormSize{}{\Theta-\trueT}-{\gamma_\alpha\over 2}\FnormSize{}{\Theta-\trueT}^2,
\end{equation}
In particular, the above inequality also holds for $\hat \Theta\in\tD$. Therefore,
\[
\tL_\tY(\hat \Theta)\leq \tL_{\tY}(\trueT)+C_2L_\alpha \left(R^{K-1}\sum_k d_k\right)^{1/2}  \FnormSize{}{\hat \Theta-\trueT}-{\gamma_\alpha\over 2} \FnormSize{}{\hat \Theta-\trueT}^2.
\]
Since $\hat \Theta=\argmax_{\Theta\in\tD}\tL_\tY(\Theta)$, $\tL_\tY(\hat \Theta)-\tL_{\tY}(\trueT)\geq 0$, which gives
\[
C_2L_\alpha \left(R^{K-1}\sum_k d_k\right)^{1/2}  \FnormSize{}{\hat \Theta-\trueT}-{\gamma_\alpha\over 2}\FnormSize{}{\hat \Theta-\trueT}^2\geq 0.
\]
Henceforth, 
\[
{1\over \sqrt{\prod_k d_k}} \FnormSize{}{\hat \Theta-\trueT}\leq {2C_2L_\alpha \sqrt{R^{K-1}\sum_k d_k}\over \gamma_\alpha \sqrt{\prod_k d_k}}=2C_2{L_\alpha\over \gamma_\alpha} \sqrt{ R^{K-1}\sum_k d_k \over \prod_k d_k}.
\]

\end{proof}

\begin{rmk}\label{rmk:level}
Based on the proof of Theorem~\ref{thm:rate}, we can relax the global optimum assumption on the estimator $\hat \Theta$. The same convergence rate holds in the level set $\{\hat \Theta\in\tD\colon \tL_\tY(\hat \Theta)\geq \tL_{\tY}(\trueT)\}$. 
\end{rmk}

\subsection{Proof of Theorem~\ref{thm:minimax}}
\label{sec:proof-thm2}

\begin{proof}
Without loss of generality, we assume $d_1=d_{\max}$, and denote by $d_{\text{total}}=\prod_{k\geq 1}d_k$. Let $\gamma\in[0,1]$ be a constant to be specified later.  Our strategy is to construct a finite set of tensors $\tX=\{\Theta_i\colon i=1,\ldots \}\subset \tD(R,\alpha)$ satisfying the properties of (i)-(iv) in Lemma~\ref{lem:construction}. By Lemma~\ref{lem:construction}, such a subset of tensors exist. For any given tensor  $\Theta\in\tX$, let $\mathbb{P}_{\Theta}$ denote the distribution of $\tY|\Theta$, where $\tY$ is the observed binary tensor. In particular, $\mathbb{P}_{\mathbf{0}}$ is the distribution of $\tY$ induced by the zero parameter tensor $\mathbf{0}$; i.e., the distribution of $\tY$ conditional on the coefficient tensor $\Theta=\mathbf{0}$. Then conditioning on $\Theta\in\tX$, the entries of $\tY$ are independent Bernoulli random variables. In addition, we note that (c.f.\  Lemma~\ref{lem:KL}), 
\begin{align}\label{eq:KL-link}
\textrm{for the logistic link:} \quad &  \text{KL}(\mathbb{P}_{\Theta}, \mathbb{P}_{\mathbf{0}}) \leq {4\over \sigma^2}\FnormSize{}{\Theta}^2, \nonumber \\
\text{for the probit link:}  \quad&  \text{KL}(\mathbb{P}_{\Theta}, \mathbb{P}_{\mathbf{0}}) \leq {2\over \pi \sigma^2}\FnormSize{}{\Theta}^2, \\
\text{for the Laplacian link:}  \quad&  \text{KL}(\mathbb{P}_{\Theta}, \mathbb{P}_{\mathbf{0}}) \leq {1\over \sigma^2}\FnormSize{}{\Theta}^2,  \nonumber
\end{align}
where $\sigma$ is the scale parameter. Therefore, under these link functions, the KL divergence between $\mathbb{P}_{\Theta}$ and $\mathbb{P}_{\mathbf{0}}$ satisfies 
\begin{equation}\label{eq:KL}
\text{KL}(\mathbb{P}_{\Theta},\mathbb{P}_{\mathbf{0}})\leq {2\over \pi \sigma^2}\FnormSize{}{\Theta}^2\leq {2\over \pi }Rd_1\gamma^2,
\end{equation}
where the first inequality comes from \eqref{eq:KL-link}, and the second inequality comes from property (iii) of $\tX$. From~\eqref{eq:KL} and the property (i), we conclude that the inequality 
\begin{equation}\label{eq:totalKL}
{1\over \text{Card}(\tX)-1}\sum_{\Theta \in\tX}\text{KL}(\mathbb{P}_{\Theta}, \mathbb{P}_{\mathbf{0}})\leq \varepsilon \log\left\{\text{Card}(\tX)-1 \right\}
\end{equation}
is satisfied for any $ \varepsilon \geq 0$, when$\gamma\in[0,1]$ is chosen to be sufficiently small depending on $\varepsilon$, e.g., $\gamma \leq \sqrt{\varepsilon \log 2 \over 8}$. By applying \citet[Theorem 2.5]{tsybakov2009introduction} to~\eqref{eq:totalKL}, and in view of the property (iv), we obtain that 
\begin{equation}\label{eq:final}
\inf_{\hat \Theta}\sup_{\trueT\in \tX}\mathbb{P}\left(\FnormSize{}{\hat \Theta- \trueT}\geq  {\gamma\over 8} \min\left\{ \alpha\sqrt{d_{\text{total}}}, \sigma\sqrt{Rd_1}\right\} \right)\geq {1\over 2}\left(1-2\varepsilon-\sqrt{16 \varepsilon \over Rd_1\log 2}\right).
\end{equation}
Note that $\FnormSize{}{\hat \Theta- \trueT}=\sqrt{d_{\text{total}}}\text{Loss}(\hat \Theta, \trueT)$ and $\tX\subset \tD(R,\alpha)$. By taking $\varepsilon=1/10$ and $\gamma=1/11$, we conclude from~\eqref{eq:final} that
\begin{align}
\inf_{\hat \Theta}\sup_{\trueT\in \tD(R,\alpha)}\mathbb{P}\left(\text{Loss}(\hat \Theta, \trueT)\geq {1\over 88^2}\min\left \{ \alpha, \sigma \sqrt{Rd_{\max}\over d_{\text{total}}}\right \}\right)&\geq {1\over 2}\left({4\over 5}-\sqrt{1.6\over Rd_{\max} \log 2}\right),
\end{align}
which is $\geq1/8$.
\end{proof}

\subsection{Proof of Theorem~\ref{thm:ratereal}}
\label{sec:proof-thm3}

\begin{proof}
The argument is similar as that in the proof of Theorem~\ref{thm:minimax}. Specifically, we construct a set of tensors $\tX\subset \tD(R,\alpha)$ such that, for all $\Theta\in\tX$, $\Theta$ satisfies the properties (i) to (iv) of Lemma~\ref{lem:construction}. Given a continuous-valued tensor $\tY$, let $\mathbb{P}_{\Theta}$ denote the distribution of $\tY|\Theta$ according to the Gaussian model; that is, $
\tY=\entry{y_{i_1,\ldots,i_K}}|\Theta\sim_{\text{i.i.d.}} N(0,\sigma^2)$. Note that, for the Gaussian distribution,
\[
KL(\mathbb{P}_{\Theta}, \mathbb{P}_{\bf 0})={\FnormSize{}{\Theta}\over 2\sigma^2}\leq {1\over 16}Rd_1\gamma^2.
\]
So the condition
\begin{equation}\label{eq:realKL}
{1\over \text{Card}(\tX)-1}\sum_{\Theta \in\tX}\text{KL}(\mathbb{P}_{\Theta}, \mathbb{P}_{\mathbf{0}})\leq \varepsilon \log\left(\text{Card}(\tX)-1 \right)
\end{equation}
is satisfied for any $ \varepsilon \geq 0$ when $\gamma\in[0,1]$ is chosen to be sufficiently small depending on $\varepsilon$. In view of the property (iv) and~\eqref{eq:realKL}, the conclusion follows readily from the application of \citet[Theorem 2.5]{tsybakov2009introduction}. 
\end{proof}

\subsection{Proof of Proposition~\ref{prop:alg}}

\begin{proof}The proof of the global convergence is similar to that of \citet[Proposition 1]{zhou2013tensor}. We present the main ideas here for completeness. By Assumption (A2), the block update is well-defined and differentiable. The isolation of stationary points ensures that there are only finite number of stationary points. It suffices to show that every sub-sequence of $\mA^{(t)}$ convergences to a same limiting point.

Let $\mA^{(t_n)}$ be one subsequence with limiting point $\mA^*$. We aim to show that $\mA^*$ is the only limiting point for all possible subsequences in $\mA^{(t)}$. As the algorithm monotonically increases the objective value, the limiting point $\mA^*$ is a stationary point of $\tL$. Now take the set of all limiting points, which is contained in the set $\{\mA\colon \tL(\mA)\geq \tL(\mA^{(0)})\}$, and is thus compact due to (A1). The compactness of the set of limiting points implies that the set is also connected~\cite[Propositions 8.2.1 and 15.4.2]{lange2010numerical}. Note that a connected subset of the finite stationery points is a single point. Henceforth, every subsequence of $\mA^{(t)}$ convergences to a stationary point of $\tL$.

The local convergence follows from~\citet[Theorem 3.3]{uschmajew2012local} and \citet[Proposition 1]{zhou2013tensor}. Here we elaborate on the contraction parameter $\rho\in(0,1)$ in our context. Let $\mH$ denote the Hession matrix of the log-likelihood $\tL(\mA)$ at the local maximum $\mA^*$. We partition the Hession into $\mH=\mL+\mD+\mL^T$, where $\mL$ is the strictly block lower triangular part and $\mD$ is the block diagonal part. By Assumption (A2), each sub-block of the Hession is negative definite, so the diagonal entries of $\mD$ are strictly negative. This ensures that the block lower triangular matrix $\mL+\mD$ is invertible. The differential of the iteration map $\tM\colon \mA^{(t)}\mapsto \mA^{(t+1)}$ can be shown as $\tM'=-(\mL+\mD)^{-1}\mL$ \citep[Lemma 2]{bezdek2003convergence}. Therefore $\rho=\max_i \left| \lambda_i\left\{ (\mL+\mD)^{-1}\mL \right\} \right| \in(0,1)$, where $\lambda_i\{\cdot\}$ denotes the $i$-th singular value of the matrix. By the contraction principle,
\[
\FnormSize{}{\mA^{(t)}-\mA^*}\leq \rho^t \FnormSize{}{\mA^{(0)}-\mA^*},
\]
for $\mA^{(0)}$ sufficiently close to $\mA^*$. Because $\Theta=\Theta(\mA)$ is local Lipschitz at $\mA^*$ with constants $c_1,c_2>0$, we have
\[
c_1\FnormSize{}{\mA^{(t)}-\mA^*} \leq  \FnormSize{}{\Theta(\mA^{(t)})-\Theta(\mA^*)} \leq c_2\FnormSize{}{\mA^{(t)}-\mA^*},
\]
for all sufficiently large $t\in\mathbb{N}_{+}$. Therefore 
\[
\FnormSize{}{\Theta(\mA^{(t)})-\Theta^*}\leq \rho^t C \FnormSize{}{\Theta(\mA^{(0)})-\Theta^*},
\]
where $C>0$ is a constant.  
\end{proof}

\subsection{Proof of Theorem~\ref{thm:empirical}}
\begin{proof}
Based on Remark~\ref{rmk:level} after Theorem~\ref{thm:rate}, we have
\begin{equation}
\text{Loss}(\Theta^*, \trueT)\leq {C_2 L_\alpha\over \gamma_\alpha}\sqrt{R^{K-1}\sum_k d_k\over \prod_k d_k}.
\end{equation}
Meanwhile, Proposition~\ref{prop:alg} implies that, there exists an iteration number $T_0\geq 0$, such that
\begin{equation}
\text{Loss}( \Theta^{(t)}, \Theta^* ) \leq C_1 \rho^t \text{Loss}( \Theta^{(0)}, \Theta^* ),
\end{equation}
holds for all $t\geq T_0$. Combining the above two results yields  
\begin{align}
\text{Loss}(\Theta^{(t)}, \trueT) &\leq \text{Loss}(\Theta^{(t)}, \Theta^*) + \text{Loss}(\trueT,\Theta^*)\\
& \leq C_1 \rho^t \text{Loss}(\Theta^{(0)}, \Theta^*)+  {C_2 L_\alpha\over \gamma_\alpha}\sqrt{R^{K-1}\sum_k d_k\over \prod_k d_k},
\end{align}
for all $t\geq T_0$.
\end{proof}

\section{Conclusions}
\label{sec:conclusion}

Many data tensors consist of binary observations. This article presents a general method and the associated theory for binary tensor decomposition. We have shown that the unknown parameter tensor can be accurately and efficiently recovered under suitable assumptions. When the maximum norm of the unknown tensor is bounded by a constant, our error bound is tight up to a constant and matches with the best possible error bound for the unquantized observations. 

We comment on a number of possible extensions. Our method leverages on the alternating updating algorithm for the optimization. Although a non-convex optimization procedure such as Algorithm~\ref{alg:binary} has no guarantee on global optimality, our numerical experiments have suggested that, upon random initializations, the convergence point $\hat\Theta$ is often satisfactory, in that the corresponding objective value $\tL_{\tY}(\hat\Theta)$ is close to the objective value $\tL_{\tY}(\trueT)$. We have shown in Theorem~\ref{thm:rate} that the same statistically optimal convergence rate holds, not only for the MLE, but also for every local maximizer $\hat\Theta$ with sufficiently large objective values. When starting from random initializations, there could be multiple estimates, whose objective values are all greater than $\tL_{\tY}(\trueT)$. In theory, any of those choices perform equally well in estimating $\trueT$. In this sense, local optimality is not necessarily a severe concern in our context. On the other hand, characterizing global optimality for non-convex optimization problem of this type is itself of great interest. There has been recent progress investigating the landscape of non-convex optimization involving tensors~\citep{anandkumar2014tensor, richard2014statistical, ge2017optimization}. The problem is challenging, as the geometry can  depend on multiple factors including the tensor rank, dimension, and factorization form. In some special cases such as rank-1 or orthogonally decomposable tensors, one may further obtain the required asymptotical number of initializations, however, at the cost of more stringent assumptions on the target tensor~\citep{anandkumar2014tensor, richard2014statistical}. We leave the pursuit of optimization landscape as future research.

For the theory, we assume the true rank $R$ is known, whereas for the application, we propose to estimate the rank using BIC given the data. It remains an open and challenging question to establish the convergence rate of the estimated rank~\citep{zhou2013tensor}. We leave a full theoretical investigation of the rank selection consistency and the decomposition error bound under the estimated rank as future research.

Finally, although we have concentrated on the Bernoulli distribution in this article, we may consider extensions to other exponential-family distributions, for example, count-valued tensors, multinomial-valued tensors, or tensors with mixed types of entries. Moreover, our proposed method can be thought of as a building block for more specialized tasks such as exploratory data analysis, tensor completion, compressed object representation, and network link prediction. Exploiting the benefits and properties of binary tensor decomposition in each specialized task warrants future research.

\section*{Acknowledgments}
Wang's research was partially supported by NSF grant DMS-1915978 and Wisconsin Alumni Research Foundation. Li's research was partially supported by NSF grant DMS-1613137 and NIH grants R01AG034570 and R01AG061303. The authors thank the Editor and three referees for their constructive comments.

\newpage

\appendix
\addcontentsline{toc}{section}{Appendices}
\renewcommand{\thesubsection}{\Alph{subsection}}

\section{Technical Lemmas}\label{sec:lemma}

We summarize technical lemmas that are useful for the proofs of the main theorems.

\begin{lem} \label{lem:inq}
Let $\tA, \tB$ be two order-$K$ tensors of the same dimension. Then,
\[ 
|\langle \tA,\tB\rangle| \leq \normSize{}{\tA}   \nnormSize{}{\tB}.
\]
\end{lem}

\begin{proof}
By~\citet[Proposition 3.1]{friedland2018nuclear}, there exists a nuclear norm decomposition of $\tB$, such that
\[
\tB=\sum_{r} \lambda_r \ma^{(1)}_r\otimes \cdots\otimes \ma^{(K)}_r,\quad \ma_r^{(k)}\in\mathbf{S}^{d_k-1}(\mathbb{R}),\quad \text{for all }k\in[K],
\]
and $\nnormSize{}{\tB}=\sum_{r}|\lambda_r|$. Henceforth we have
\begin{align}
|\langle \tA,\tB\rangle|&=| \langle \tA, \sum_{r} \lambda_r \ma^{(1)}_r\otimes \cdots\otimes \ma^{(K)}_r \rangle|\leq \sum_r |\lambda_r| |\langle \tA, \ma^{(1)}_r \otimes \cdots\otimes \ma^{(K)}_r \rangle|\\
&\leq \sum_{r}|\lambda_r| \normSize{}{\tA}= \normSize{}{\tA}\nnormSize{}{\tB}.
\end{align}
\end{proof}

\begin{lem} \label{lem:nuclear}
Let $\tA\in\mathbb{R}^{d_1\times\cdots\times d_K}$ be an order-$K$ tensor with $\text{rank}(\tA)\leq R$. Then,
\[
\nnormSize{}{\tA} \leq R^{{K-1\over 2}} \FnormSize{}{\tA},
\]
where $\nnormSize{}{\cdot}$ denotes the nuclear norm of the tensor. 
\end{lem}

\begin{proof}
Let $\text{rank}(\cdot)$ denote the regular matrix rank, and $\tA_{(k)}$ denote the mode-$k$ matricization of $\tA$, $k\in[K]$. Define $\text{rank}_T(\tA)=(R_1,\ldots,R_K)$ as the Tucker rank of $\tA$, with $R_k=\text{rank}(\tA_{(k)})$. The condition $\text{rank}(\tA)\leq R$ implies that $R_k\leq R$ for all $k\in[K]$. Without loss of generality, assume $R_1=\min_{k}R_k$. By \citet[Corollary 4.11]{wang2017operator} and the invariance relationship between a tensor and its Tucker core~\citep[Section 6]{jiang2017tensor}, we have
\begin{equation}\label{eq:norminequality}
\nnormSize{}{\tA} \leq \sqrt{\prod_{k=2}^K R_k \over \max_{k\geq 2} R_k} \nnormSize{}{\tA_{(1)}} \leq R^{{K-2\over 2}} \nnormSize{}{\tA_{(1)}},
\end{equation}
where $\tA_{(1)}$ is a $d_1$-by-$\prod_{k\geq 2}d_k$ matrix with rank bounded by $R$. Furthermore, the relationship between the matrix norms implies that $\nnormSize{}{\tA_{(1)}}\leq \sqrt{R}\FnormSize{}{\tA_{(1)}}=\sqrt{R}\FnormSize{}{\tA}$. Combining this fact with the inequality~\eqref{eq:norminequality} yields the final claim. 
\end{proof}

\begin{lem} \label{lem:KL}
Let $\tY\in\{0,1\}^{d_1\times \cdots \times d_K}$ be a binary tensor. Let $\mathbb{P}_{\Theta}$ denote the distribution of $\tY|\Theta$ based on the Bernoulli model~\eqref{eq:model} with the link function $f$ and the parameter tensor $\Theta$. Let $\mathbb{P}_{\mathbf{0}}$ denote the distribution of $\tY|\mathbf{0}$ induced by the zero parameter tensor. Then
\[
\text{KL}(\mathbb{P}_{\Theta}, \mathbb{P}_{\mathbf{0}})\leq 4\dot{f}^2(0)\FnormSize{}{\Theta}^2.
\]

\begin{proof} 
We have that 
\begin{align}
\text{KL}(\mathbb{P}_{\Theta}, \mathbb{P}_{\mathbf{0}}) &= \sum_{i_1,\ldots,i_K}\text{KL}(\tY_{i_1,\ldots,i_K}|\theta_{i_1,\ldots,i_K}, \tY_{i_1,\ldots,i_K}|0) 
\leq \sum_{i_1,\ldots,i_K}{\left(f(\theta_{i_1,\ldots,i_K})-f(0)\right)^2\over f(0)(1-f(0))}\\
 &=  \sum_{i_1,\ldots,i_K}{\dot{f}^2(\eta_{i_1,\ldots,i_K}\theta_{i_1,\ldots,i_K})\left(\theta_{i_1,\ldots,i_K}-0\right)^2\over f(0)(1-f(0))} 
\leq \sum_{i_1,\ldots,i_K}4 \dot{f}^2(0)\theta_{i_1,\ldots,i_K}^2\\
 &=  4\dot{f}^2(0) \FnormSize{}{\Theta}^2,
\end{align}
where the first inequality comes from Lemma~\ref{lem:KL0}, the next equality comes from the first-order Taylor expansion with $\eta_{i_1,\ldots,i_K}\in[0,1]$, and the last inequality uses the fact that $f(0)=1/2$ and $f'$ peaks at zero for an unimodal and symmetric density function. 
\end{proof}
\end{lem}

\begin{lem} \label{lem:KL0}
Let $X, Y$ be two Bernoulli random variables with means $p$ and $q$, $0< p,q<1$, respectively. Then, the Kullback-Leibler (KL) divergence satisfies that 
\[
KL(X,Y)\leq {(p-q)^2\over q(1-q)},
\]
where $\text{KL}(X,Y)=-\sum_{x=\{0,1\}}P_{X}(x)\log \left\{{P_Y(x)\over P_X(x)}\right\}$.
\end{lem}
\begin{proof} It is straightforward to verify that
\[
KL(X,Y)=p\log{p\over q}+(1-p)\log{1-p\over 1-q} \; \leq \; p{p-q\over q}+(1-p){q-p\over 1-q}={(p-q)^2\over q(1-q)},
\]
where the inequality is due to the fact that $\log x\leq x-1$ for $x>0$. 
\end{proof}

\begin{lem}[\cite{tomioka2014spectral}]\label{lem:tensor}
Suppose that $\tS=\entry{s_{i_1,\ldots,i_K}}\in\mathbb{R}^{d_1\times \cdots \times d_K}$ is an order-$K$ tensor whose entries are independent random variables that satisfy
\[
\mathbb{E}(s_{i_1,\ldots,i_K})=0,\quad \text{and} \quad\mathbb{E}(e^{ts_{i_1,\ldots,i_K}})\leq e^{t^2L^2/2}.
\]
Then, the spectral norm $\normSize{}{\tS}$ satisfies that
\[
\normSize{}{\tS}\leq \sqrt{{8L^2} \log (12K) \sum_k d_k +\log (2/\delta)},
\]
with probability at least $1-\delta$.
\end{lem}

\begin{rmk}
The above lemma provides the bound on the spectral norm of random tensors. Similar results were presented in~\cite{nguyen2015tensor}, and we adopt the version from~\cite{tomioka2014spectral}.  
\end{rmk}

\begin{lem} \label{lem:noisytensor}
Suppose that $\tS=\entry{s_{i_1,\ldots,i_K}}\in\mathbb{R}^{d_1\times \cdots \times d_K}$ is an order-$K$ tensor whose entries are independent random variables that satisfy
\[
\mathbb{E}(s_{i_1,\ldots,i_K})=0\quad \text{and}\quad |s_{i_1,\ldots,i_K}|\leq L.
\]
Then, we have
\[
\mathbb{P}\left(\normSize{}{\tS}\geq C_2 L\sqrt{\sum_k d_k} \right)\leq \exp\left(-C_1  \log K \sum_k d_k\right),
\]
where $C_1>0$ is an absolute constant, and $C_2>0$ is a constant that depends only on $K$. 
\end{lem}

\begin{proof}  Note that the random variable $L^{-1}s_{i_1,\ldots,i_K}$ is zero-mean and supported on $[-1,1]$. Therefore, $L^{-1}s_{i_1,\ldots,i_K}$ is sub-Gaussian with parameter ${1-(-1)\over 2}=1$; i.e.,
\[
\mathbb{E}(L^{-1}s_{i_1,\ldots,i_K})=0\quad \text{and}\quad \mathbb{E}(e^{tL^{-1}s_{i_1,\ldots,i_K}})\leq e^{t^2/2}.
\]
It follows from Lemma~\ref{lem:tensor} that, with probability at least $1-\delta$, 
\[
\normSize{}{L^{-1}\tS}\leq \sqrt{\left(c_0\log K+c_1\right) \sum_k d_k +\log (2/\delta)},
\]
where $c_0, c_1>0$ are two absolute constants. Taking $\delta=\exp (-C_1\log K \sum_k d_k)$ yields the final claim, where $C_2=c_0\log K+c_1+1>0$ is another constant.  
\end{proof}

\begin{lem}[Varshamov-Gilbert bound]\label{lem:VGbound}
Let $\Omega=\{(w_1,\ldots,w_m)\colon w_i\in\{0,1\}\}$. Suppose $m>8$. Then, there exists a subset $\{w^{(0)},\ldots,w^{(M)}\}$ of $\Omega$ such that $w^{(0)}=(0,\ldots,0)$ and
\[
\zeronormSize{}{w^{(j)}-w^{(k)}}\geq {m\over 8},\quad \text{for} \ 0\leq j<k\leq M,
\]
where $\zeronormSize{}{\cdot}$ denotes the Hamming distance, and $M\geq 2^{m/8}$. 
\end{lem}

\begin{lem}\label{lem:construction}
Assume the same setup as in Theorem~\ref{thm:minimax}. Without loss of generality, suppose $d_1=d_{\max}$, and define $d_{\text{total}}=\prod_{k\geq 1} d_k$. For any given constant $0\leq \gamma \leq 1$, there exist a finite set of tensors $\tX=\{\Theta_i\colon i=1,\ldots\}\subset \tD(R,\alpha)$ satisfying the following four properties:
\begin{enumerate}[(i)]
\item $\text{Card}(\tX)\geq 2^{Rd_1/8}+1$, where $\text{Card}(\cdot)$ denotes the cardinality of the set;
\item $\tX$ contains the zero tensor $\mathbf{0}\in\mathbb{R}^{d_1\times \cdots\times d_K}$;
\item $\mnormSize{}{\Theta}\leq \gamma \min\left\{ \alpha , \sigma \sqrt{Rd_1\over d_{\text{total}}} \right\} $ for all elements $\Theta\in\tX$;
\item $\FnormSize{}{\Theta_i-\Theta_j}\geq {\gamma\over 4} \min\left\{ \alpha\sqrt{d_{\text{total}}}, \sigma\sqrt{Rd_1}\right\}$ for any two distinct elements $\Theta_i\neq \Theta_j\in\tX$. 
\end{enumerate}
\end{lem}
\begin{rmk} Lemma~\ref{lem:construction} is a special case of~\citet[Lemma 9]{lee2020tensor}. We provide the proof here for completeness. 
\end{rmk}
\begin{proof}
Given a constant $0\leq \gamma \leq 1$, we define a set of matrices,
\[
\tC=\left\{\mM=(m_{ij})\in\mathbb{R}^{d_1\times R}\colon a_{ij}\in \left\{ 0,\gamma \min\left\{ \alpha , \sigma \sqrt{Rd_1\over d_{\text{total}}} \right\}\right\} ,\  \forall (i,j)\in[d_1]\times[R]\right\}.
\]

We then consider the associated set of block tensors,
\begin{align}
\tB=\tB(\tC)=\{\Theta\in\mathbb{R}^{d_1\times \cdots \times d_K}\colon& \Theta=\mA\otimes \mathbf{1}_{d_3}\otimes \cdots \otimes \mathbf{1}_{d_K}, \\
 &\text{where}\ \mA=(\mM|\cdots|\mM|\mO) \in\mathbb{R}^{d_1\times d_2},\ \mM\in\tC\},
\end{align}
where $\mathbf{1}_d$ denotes a length-$d$ vector with all entries 1, $\mO$ denotes the $d_1\times (d_2-R\lfloor d_2/R \rfloor)$ zero matrix, and $\lfloor d_2/ R \rfloor$ is the integer part of $d_2/R$. In other words, the subtensor $\Theta(\mI, \mI,i_3, \ldots,i_K)\in\mathbb{R}^{d_1\times d_2}$ are the same for all fixed $(i_3,\ldots,i_K)\in[d_3]\times \cdots \times [d_K]$, and furthermore, each subtensor $\Theta(\mI,\mI, i_3,\ldots,i_K)$ itself is filled by copying the matrix $\mM\in\mathbb{R}^{d_1\times R}$ as many times as would fit.

By construction, all tensors in $\tB$, as well as the difference of any two tensors in $\tB$, has tensor rank at most $R$. Furthermore, the entrywise magnitudes of tensor entries in $\tB$ are bounded by $\alpha$. Thus, $\tB\subset\tD(R,\alpha)$. By Lemma~\ref{lem:VGbound}, there exists a subset $\tX\subset \tB$ with cardinality $\text{Card}(\tX)\geq 2^{d_1R/8}+1$ containing the zero $d_1\times \cdots \times d_K$ tensor, such that, for any two distinct elements $\Theta_i$ and $\Theta_j$ in $\tX$, 
\[
\FnormSize{}{\Theta_i-\Theta_j}^2 \geq {d_1R\over 8} \gamma^2\min\left\{ \alpha, {\sigma^2Rd_1 \over d_{\text{total}}}\right\} \lfloor {d_2\over R} \rfloor \prod_{k\geq 3}d_k\geq {\gamma^2\min\left\{ \alpha^2 d_{\text{total}}, \sigma^2Rd_1\right\}  \over 16}.
\]
In addition, each entry of $\Theta\in\tX$ is bounded by $\gamma \min\left\{ \alpha , \sigma \sqrt{Rd_1\over d_{\text{total}}}\right\} $. Therefore, the Properties (i)--(iv) are satisfied. 
\end{proof}

\bibliographystyle{apalike}
\bibliography{tensor_wang.bib}

\end{document}